\documentclass[letterpaper]{article} % DO NOT CHANGE THIS
\usepackage{aaai2026}  % DO NOT CHANGE THIS
\usepackage{times}  % DO NOT CHANGE THIS
\usepackage{helvet}  % DO NOT CHANGE THIS
\usepackage{courier}  % DO NOT CHANGE THIS
\usepackage[hyphens]{url}  % DO NOT CHANGE THIS
\usepackage{graphicx} % DO NOT CHANGE THIS
\usepackage{natbib}  % DO NOT CHANGE THIS AND DO NOT ADD ANY OPTIONS TO IT
\usepackage{caption} % DO NOT CHANGE THIS AND DO NOT ADD ANY OPTIONS TO IT
\usepackage{graphics}
\usepackage{subfigure}
\usepackage[linesnumbered,ruled,vlined]{algorithm2e}
\SetKwComment{Comment}{$\triangleright$~}{}
\usepackage{amsmath}
\usepackage{amssymb}
\usepackage{mathtools}
\usepackage{amsthm}
\usepackage{xspace}
\usepackage{epstopdf}
\usepackage{amsmath}
\usepackage{enumitem}
\setitemize{topsep=2pt,parsep=2pt,partopsep=2pt,leftmargin=20pt}
\usepackage{colortbl}
\usepackage{amssymb}

\usepackage{bbding}
\usepackage{pifont}
\usepackage{wasysym}
\usepackage{multirow}
\usepackage{multicol}
\theoremstyle{plain}
\newtheorem{theorem}{Theorem}[section]
\newtheorem{proposition}[theorem]{Proposition}
\newtheorem{lemma}[theorem]{Lemma}

\theoremstyle{definition}

\newtheorem{assumption}[theorem]{Assumption}
\theoremstyle{remark}

\newcommand{\SysName}{\mbox{SMoFi}\xspace}
\usepackage{caption}
\usepackage{booktabs}

\usepackage{etoolbox}
\AtBeginEnvironment{align}{\setlength{\parindent}{0pt}}

\allowdisplaybreaks % Allows page breaks in displays

\title{SMoFi: Step-wise Momentum Fusion for Split Federated Learning on Heterogeneous Data}
\author{ 
Mingkun Yang\textsuperscript{}, Ran Zhu\textsuperscript{}, Qing Wang\textsuperscript{}, Jie Yang\textsuperscript{}
}
\affiliations{
    \textsuperscript{}Delft University of Technology\\
    \{m.yang-3, r.zhu-1, qing.wang, j.yang-3\}@tudelft.nl
}

\begin{document}
\maketitle

\begin{abstract}
\label{sec_abstract}

Split Federated Learning is a system-efficient federated learning paradigm that leverages the rich computing resources at a central server to train model partitions. Data heterogeneity across silos, however, presents a major challenge undermining the convergence speed and accuracy of the global model. This paper introduces \textbf{Step-wise Momentum Fusion (SMoFi)}, an effective and lightweight framework that counteracts gradient divergence arising from data heterogeneity by synchronizing the momentum buffers across server-side optimizers. To control gradient divergence over the training process, we design a staleness-aware alignment mechanism that imposes constraints on gradient updates of the server-side submodel at each optimization step. Extensive validations on multiple real-world datasets show that SMoFi consistently improves global model accuracy (up to 7.1\%) and convergence speed (up to 10.25$\times$). Furthermore, SMoFi has a greater impact with more clients involved and deeper learning models, making it particularly suitable for model training in resource-constrained contexts.
\end{abstract}

\section{Introduction}
\label{sec:intro}
The proliferation of mobile and sensing devices~\citep{gubbi2013internet, fortino2014internet} has resulted in rich data at the edge, enabling a new range of Artificial Intelligence of Things applications~\citep{wang2020convergence, verbraeken2020survey}. In this context, Federated Learning (FL) has been proposed as an important learning paradigm that exploits data generated at distributed edge devices while preserving data privacy~\citep{konevcny2016federated, mcmahan2017communication, bonawitz2019towards}. On-device training of large models, however, remains a key challenge due to limited computing resources at the edge. To expedite on-device model training, prior work has explored collaborative model training, resorting to richer computing resources ~\citep{li2018deep, eshratifar2019jointdnn, wang2021hivemind}. A promising approach is split learning~\citep{gupta2018distributed, vepakomma2018no, singh2019detailed, gao2020end, thapa2022splitfed}, which divides a model into submodels trained separately on the edge and cloud, thereby offloading partial training overhead to the powerful server.

The effectiveness of split FL, in terms of both model accuracy and convergence, is largely undermined by data heterogeneity--i.e., non-IID (non-identically and/or independently distributed) data silos--often presenting in real-world scenarios. Training on such non-IID data introduces inconsistency in model updates across both client-side and server-side submodels; such a divergence accumulates as the training proceeds over iterations. Consequently, aggregation of diverse model updates results in inferior accuracy and slower convergence of the global model compared with the IID setting. Existing methods to address such a challenge in FL either modify the loss function~\citep{li2020federated, li2021model, gao2022feddc} or impose robustness constraints on server aggregation~\citep{hsu2019measuring, wang2020tackling, reddi2020adaptive, shi2025fedlws}. While many of them are adaptable to split context, they overlook a unique attribute of split learning: the server directly controls the learning pace of multiple surrogate server-side models (each paired with a client-side submodel) that typically constitute the majority of the full model. This leads us to ask: \textit{Can we impose constraints on model training in split FL by leveraging its inherent client-server interaction--enforced more tightly and synchronously than in conventional FL--without introducing additional overheads or privacy risk?}

\begin{figure}[t]
    \centering
    \includegraphics[width=0.9\linewidth]{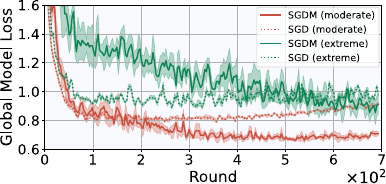}
    \vspace{-2mm}
    \caption{Momentum in local optimizers improves model performance on both moderate and extreme non-IID data in the long run, albeit slows down the learning.}
    \vspace{-4mm}
    \label{fig:motivation}
\end{figure}

In this paper, we propose \SysName, a split FL framework that fuses momentum on a step basis to reduce weight inconsistency for learning with non-IID data. Momentum, as found in previous work, can improve FL in general to reach better model performance~\citep{wangslowmo}; straightforward integration of stochastic gradient descent with momentum (SGDM) into local training of FL, however, slows down the learning as the local models converges towards their respective local optima better, making increasingly divergent updates, as shown in Figure~\ref{fig:motivation}. Such an effect is more significant when the data is more heterogeneous, e.g., green lines in the figure. In \SysName, we propose to align optimization trajectories across server-side optimizers by synchronizing their momentum buffers in every step. As a result, \SysName turns momentum from a slowing-down factor to a mechanism that speeds up model convergence while at the same time benefiting from the better performance brought about by momentum. By fusing momentum, \SysName makes minimum changes to the existing split FL framework and thus can be plugged into existing FL methods. \SysName is carefully designed to cope with a specific challenge of step-wise momentum fusion: certain devices with fewer batches finish the training earlier in the same round, not contributing to momentum alignment for model updates of the other devices. To maintain the constraint imposed by momentum alignment on the server-side models over the entire round, \SysName fuses historical momentum with a staleness factor.

In summary, we make the following key contributions:
\begin{itemize}
\item
We propose a novel split FL framework that improves the consistency of server-side submodel updates on non-IID data via step-wise momentum fusion.
\item
We introduce the staleness factor to counteract the diminishing momentum buffers synchronization to maintain the effectiveness of momentum fusion over the entire training process, especially as the number of local steps varies across participating clients.
\item
We conduct extensive evaluations and demonstrate that \SysName significantly improves both the model accuracy (up to 7.1\%) and convergence speed (up to 10.25$\times$). Furthermore, our framework provides greater improvements when the number of clients increases and for larger models, making it particularly suitable for FL in resource-constrained contexts.
\end{itemize}
\section{The \SysName Framework}
\label{sec:approach}

\begin{figure*}[t]
    \centering
    \includegraphics[width=0.9\linewidth]{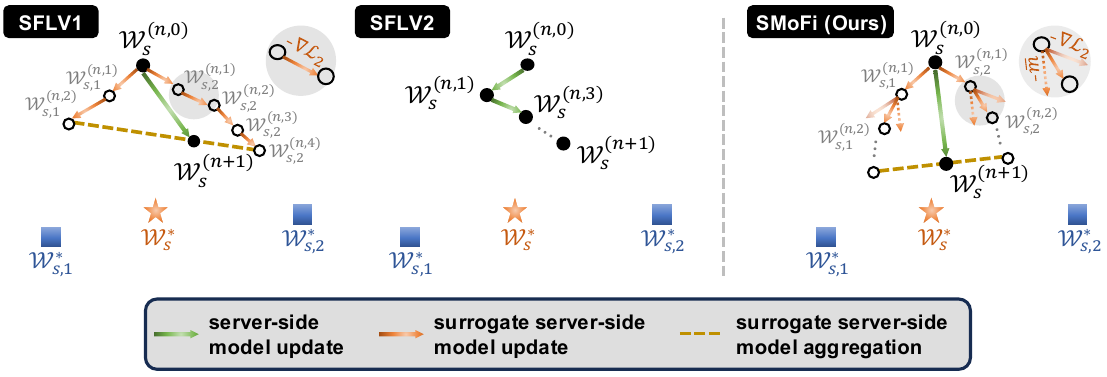}
    \vspace{-2mm}
    \caption{Comparison of server-side model updates $\mathcal{W}_{s}^{(n,0)}\mapsto\mathcal{W}_{s}^{n+1}$ in our \SysName, and the state-of-the-art SFV1 and SFV2: In SFV1, the server updates surrogate server-side models $\mathcal{W}_{s, j}^{(n,\tau)}$ in parallel, and periodically aggregates them--e.g., after each local epoch as illustrated; In SFV2, the server sequentially interacts with clients to update the server-side model; \SysName is akin to the SFV1 where the server updates surrogate models in parallel while introduces momentum alignment at each step $\tau$ by synchronizing the momentum buffers $\bar{m}$ across the server-side solvers. Such alignment helps the aggregated model converge toward the global optimum $\mathcal{W}_{s}^{*}$, rather than local optima $\mathcal{W}_{s,1}^{*}$ or $\mathcal{W}_{s,2}^{*}$.}
    \label{fig:exp_workflow}
\end{figure*}

Our \SysName is built upon the split FL framework to mitigate inconsistencies in weight updates caused by data heterogeneity, thereby facilitating more stable convergence toward the global optimum. Leveraging the modular structure of split FL, we introduce a step-wise momentum fusion as \SysName's central design: at every SGD step, momentum buffers across all server-side solvers are synchronized. This alignment of optimization trajectories imposes constant constraints on inconsistent weight evolution during the training procedure. Our \SysName also maintains a client-transparent architecture--requiring no changes or additional computation on the client side--thereby preserving the same privacy guarantees as existing frameworks such as SFLV1 and SFLV2~\citep{thapa2022splitfed}. In the following, we present the \SysName framework in detail, beginning with the general split FL structure and then presenting our proposed momentum alignment strategy.

\subsection{Collaborative Training in Split FL}
\label{subsec:collaborative_training}

In split FL, a central server with rich computing resources collaborates with a set of clients $\mathcal{J}$, each possessing local data $\mathcal{D}_{j}$ ($j \in \mathcal{J}$), to collaboratively train a task model $f^{\mathcal{W}}$. Given the index of cut layer $L$, the model is split into the client-side model $f^{\mathcal{W}_{c}}$ and the server-side model $f^{\mathcal{W}_{s}}$: $f^{\mathcal{W}}(\cdot) = f^{\mathcal{W}_{s}}(f^{\mathcal{W}_{c}}(\cdot))$, where $\mathcal{W} = [\mathcal{W}_{c}, \mathcal{W}_{s}] = [{\mathcal{W}_1, \cdots, \mathcal{W}_L}, {\mathcal{W}_{L+1}, \cdots, \mathcal{W}_{\lvert\mathcal{W}\lvert}}]$, satisfying $\mathcal{W}\in \mathbb{R}^{d}$, $\mathcal{W}_{c}\in \mathbb{R}^{d_c}$, $\mathcal{W}_{s}\in \mathbb{R}^{d_s}$, and $d=d_c+d_s$. For simplicity, we assume a fixed cut layer across all clients and communication rounds, and that the full model is split into two parts (rather than three parts that leave both the bottom and top submodels with the clients to avoid label sharing). split FL aims to find the optimal $\mathcal{W}^{*}$ as in FL 
\begin{equation}
\mathcal{W}^{*} = 
\mathop{\arg\min}\limits_{\mathcal{W}}\mathcal{L}(\mathcal{W})  = 
\mathop{\arg\min}\limits_{\mathcal{W}} \sum_{j\in \mathcal{J}}{p}_{j}\mathcal{L}_{\mathcal{D}_{j}}(\mathcal{W}),
\label{eq:approach_1}
\end{equation}
wherein, the global objective function $\mathcal{L}(\mathcal{W})$ is the weighted sum of local objectives $\{\mathcal{L}_{\mathcal{D}_{j}}(\mathcal{W})\}_{j\in \mathcal{J}}$, with weights satisfying $\sum_{j\in \mathcal{J}} {p}_{j} = 1$, e.g., the fraction of local samples. The optimal global model parameters $\mathcal{W}^{*}$ are approached by having each client optimize its local objective and then aggregating these local model parameters, iterating over the $N$ communication rounds. 

In round $n\in[N]$ ($[N]=\{1, \cdots, N\}$), considering the communication bandwidth and client availability, the central server randomly selects a subset of devices $\mathcal{J}^{n}\subseteq\mathcal{J}$ to perform SGD updating the $\mathcal{W}_{c}$ and $\mathcal{W}_{s}$ over the $\{T_{j}\}_{j\in \mathcal{J}^{n}}$ local steps. The number of steps $T_{j} = E \lfloor \frac{\lvert \mathcal{D}_{j} \lvert}{B} \rfloor$ ( $\lfloor \cdot \rfloor$ is the floor function) depending on local epochs $E$, mini-batch size $B$, and the size of local samples $\mathcal{D}_{j}$ that varies across clients. 

\noindent{\textbf{One-step SGD.}} All the participating clients perform one step of SGD in parallel. The $j$-th client bootstraps the local SGD by propagating forward on a randomly selected sample batch. It then offloads activations--also called \textit{smashed data} at the layer $L$--to the server. We denote activations as $\mathbf{A}^{(n, \tau)}_{j} = \{f^{\mathcal{W}_{c}}(\mathbf{x})\}_{\mathbf{x}\in\mathcal{B}^{\tau}_{j}}$ where $\tau\in [T_{j}]$ is the SGD step index and $\mathcal{B}^{\tau}_{j}\subseteq\mathcal{D}_{j}$ is the sample batch. The server proceeds with forward propagation $\hat{\mathbf{Y}} = \{f^{\mathcal{W}_{s}}(a)\}_{a\in \mathbf{A}^{(n, \tau)}_{j}}$, followed by gradient descent on the surrogate server-side model $\mathcal{W}_{s, j}^{(n, \tau+1)} = \mathcal{W}_{s, j}^{(n, \tau)} - \eta \nabla\mathcal{L}_{\mathcal{B}^{\tau}_{j}}(\mathcal{W}_{s, j}^{(n, \tau)})$ with the learning rate $\eta$. Specifically, the stochastic gradients on the server-side model are $\nabla\mathcal{L}_{\mathcal{B}^{\tau}_{j}}(\mathcal{W}_{s, j}^{(n, \tau)}) = \frac{1}{\lvert \mathcal{B}^{\tau}_{j} \lvert} \sum_{\mathbf{x}\in\mathcal{B}^{\tau}_{j}} \nabla l(\mathbf{x}; \mathcal{W}_{s, j}^{(n, \tau)})$, and $l(\cdot)$ is the loss function, e.g., cross-entropy loss between $\hat{\mathbf{Y}}$ and labels shared from clients\footnote{Real-world split FL implementation partitions the full model into three parts: predictions $\hat{\mathbf{Y}}$ from the top-submodel remain local; thus loss calculation is performed by clients without label sharing.}. The server sends the gradients on the cut layer back to client $j$, which then backpropagates through the local model following the chain rule: $\mathcal{W}_{c, j}^{(n, \tau+1)} = \mathcal{W}_{c, j}^{(n, \tau)} - \eta \nabla\mathcal{L}_{\mathcal{B}^{\tau}_{j}}(\mathcal{W}_{c, j}^{(n, \tau)})$.

\noindent{\textbf{One-step SGDM.}} When using SGDM~\citep{polyak1964some, liu2020improved} as the optimizer, the updating rule for the server-side submodel is reformulated as
\begin{equation}
m_{s, j}^{(n, \tau+1)} = 
\beta m_{s, j}^{(n, \tau)} + \nabla\mathcal{L}_{\mathcal{B}^{\tau}_{j}}(\mathcal{W}_{s, j}^{(n, \tau)}),
\label{eq:approach_2}
\end{equation}
\begin{equation}
\mathcal{W}_{s, j}^{(n, \tau+1)} = 
\mathcal{W}_{s, j}^{(n, \tau)} - \eta m_{s, j}^{(n, \tau+1)},
\label{eq:approach_3}
\end{equation}
where $\beta$ is the momentum coefficient, and $m_{s, j}^{(n, \tau)}$ is the momentum buffer retaining the gradients in the history steps. Similarly, each client updates the local submodel based on the gradients $\nabla\mathcal{L}_{\mathcal{B}^{\tau}_{j}}(\mathcal{W}_{c, j}^{(n, \tau)})$ and momentum buffer $m_{c, j}^{(n, \tau)}$.

\subsection{\SysName Design}
\label{subsec:design}

Benefiting from our client-transparent design in \SysName, we can focus on the server-side optimization. Broadly speaking, there are two main strategies for updating the server-side submodel in split FL: 1) \textit{Parallel updating}, where the server updates the $\lvert \mathcal{J}^{n} \lvert$ surrogate submodels in parallel and periodically synchronizes the submodels by weighted averaging, as in SFLV1; and 2) \textit{Sequential updating}, where the server retains a single submodel and update it by sequentially training with clients in the randomized order, as in SFLV2. As illustrated in Figure~\ref{fig:exp_workflow}, the updates of surrogate server-side models in SFLV1 are independent across local steps. Even when using SGDM as the optimizer, each surrogate converges toward a diverged local optimum, as it minimizes objective $\{\mathcal{L}_{\mathcal{B}^{\tau}_{j}}(\cdot)\}_{j\in\mathcal{J}^{n}}$ inconsistent across clients due to heterogeneous data silos and varying local steps $\{T_{j}\}_{j\in \mathcal{J}^{n}}$~\citep{wang2020tackling}. Besides, the sequential training in SFLV2 introduces severe latency in practical deployment. Considering both efficiency and efficacy, \SysName adopts parallel updates of the surrogate server-side submodels, while imposing consistency constraints through momentum alignment to mitigate divergent convergence trajectories during collaborative training in split FL.

\noindent{\textbf{Momentum Alignment.}} In the \textit{Parallel updating} framework, the server maintains $\lvert \mathcal{J}^{n} \lvert$ optimizers where each optimizer has a momentum buffer $m_{s, j}^{(n, \tau)}$ tracking the accumulated gradients of the $j$-th server-side model. Under objective inconsistency, these momentum buffers $\{m_{s, j}^{(n, \tau)}\}_{j\in\mathcal{J}^{n}}$ gradually diverge as training proceeds. The momentum buffers then further negatively affect gradient descent, causing the models across clients to converge in increasingly divergent directions. To address this, \SysName aligns the momentum buffers across the server-side optimizers, and Equation~\ref{eq:approach_2} is then reformulated as
\begin{equation}
m_{s, j}^{(n, \tau+1)} = 
\beta \bar{m}_{s}^{(n, \tau)} + \nabla\mathcal{L}_{\mathcal{B}^{\tau}_{j}}(\mathcal{W}_{s, j}^{(n, \tau)}),
\label{eq:approach_4}
\end{equation}
where $\bar{m}_{s}^{(n, \tau)}$ is the aligned momentum, synchronized after all server-side solvers perform gradient descent at each local step. In this way, each server-side model is updated based on stochastic gradients over its local mini-batch samples, combined with a unified momentum term (\textbf{line 7} in Algorithm~\ref{algo:pseudocode_1}).

\SysName synchronizes the momentum buffer by weighted averaging the momentum in all server-side solvers at each local step. However, as training progresses, the number of active server-side solvers contributing to this average decreases because some surrogate models complete their training earlier than others. Specifically, the set of active clients at the $\tau$-th step, denoted by $\mathcal{J}^{(n, \tau)} \subseteq \mathcal{J}^{n}$, includes only those clients for which $T_j \geqslant \tau$. Given that $T_j$ varies across the clients due to non-IID local data, the size of the active set $\lvert \mathcal{J}^{(n, \tau)} \lvert$ tends to decrease over local steps. It leads to fewer solver momentum being averaged as training progresses, thereby diminishing the strength of the constraint from momentum alignment.

\noindent{\textbf{Staleness Factor.}} To maintain the effectiveness of constraint throughout all $max\{T_j\}_{j \in \mathcal{J}^{n}}$ local steps, \SysName introduces a mechanism that leverages $\mathcal{H}^{(n)}$, a record of the momentum buffers $m_{s, j}^{(n, \tau + 1)}$ at the final step of the client, i.e., when $\tau = T_j $, $\forall j \in \mathcal{J}^{(n, \tau)}$.  At each local step, \SysName aligns the momentum buffer for the next step (\textbf{line 12} in Algorithm~\ref{algo:pseudocode_1}) by averaging both the current momentum of optimizers and the historical state stored in $\mathcal{H}^{(n)}$:
\begin{equation}
\small
\bar{m}_{s}^{(n, \tau+1)} = 
\frac
{\sum_{j\in\mathcal{J}^{(n, \tau)}}m_{s, j}^{(n, \tau+1)} + \sum_{j\in\mathcal{H}^{n}}s_{\alpha}(\tau)m_{s, j}^{(n, \lvert T_{j} \lvert + 1)}}
{\lvert \mathcal{J}^{(n, \tau)} \lvert + \lvert \mathcal{H}^{n} \lvert},
\label{eq:approach_5}
\end{equation}
where  $s_{\alpha}$ is the staleness of the historical momentum and $\lvert \mathcal{J}^{(n, \tau)} \lvert + \lvert \mathcal{H}^{n} \lvert=\lvert \mathcal{J}^{n} \lvert$, $\forall \tau\in[max\{T_j\}_{j\in\mathcal{J}^{n}}]$. The momentums in the current step are equally important to the momentum alignment as gradients are calculated on the same size of mini-batch across surrogate server-side solvers. For the historical ones, we employ a polynomial staleness factor~\cite {xie2019asynchronous}, satisfying
\begin{equation}
s_{\alpha} = (\tau-\lvert T_{j} \lvert+1)^{\alpha}, \alpha<0.
\label{eq:approach_6}
\end{equation}
In this way, the number of momentum buffers that contribute to the synchronization remains constant $\lvert \mathcal{J}^{(n)} \lvert$ at each step.

\begin{algorithm}[t]
\small
\KwIn{Selected clients $\mathcal{J}^{n}$; Cut layer $L$; Numbers of local steps $\{T_{j}\}_{j\in\mathcal{J}^{n}}$; Current global model weights $\mathcal{W}^{(n-1)}$; Learning rate $\eta$; Momentum coefficient $\beta$.\\}
\KwOut{$\{\mathcal{W}_{s, j}^{(n, T_{j})} \}_{j\in\mathcal{J}^{n}}$\\}
$\{\mathcal{W}_{s, j}^{(n, 0)} \}_{j\in\mathcal{J}^{n}}\gets\mathcal{W}^{(n-1)}_{L:-1}$, \colorbox{pink}{$\mathcal{H}^{n}$}$\gets\emptyset$, \colorbox{pink}{$\bar{m}_{s}^{(n, 0)}$}$\gets\mathbf{0}$\\
\textbf{Server Executes:}\\
\For{step $\tau = 0, 1, \cdots, max\{T_{j}\}_{j\in\mathcal{J}^{n}}-1$}
{
    \For{$j\in \mathcal{J}^{(n, \tau)}$ \textbf{in parallel}}
    {
    \tcp*[h]{Server-side Backpropagation}\\
    $\mathbf{A}^{(n, \tau)}_{j}\gets$Collecting local activation\\
    $\nabla\mathcal{L}_{\mathcal{B}^{\tau}_{j}}(\mathcal{W}_{s, j}^{(n, \tau)})\gets$Stochastic gradients on $\mathbf{A}^{(n, \tau)}_{j}$\\
    $m_{s, j}^{(n, \tau+1)} \gets \beta$\colorbox{pink}{$\bar{m}_{s}^{(n,\tau)}$} $+\nabla\mathcal{L}_{\mathcal{B}^{\tau}_{j}}(\mathcal{W}_{s, j}^{(n, \tau)})$ \\
    $\mathcal{W}_{s, j}^{(n, \tau+1)} \gets \mathcal{W}_{s, j}^{(n, \tau)} - \eta m_{s, j}^{(n, \tau+1)}$\\
    $\nabla\mathcal{L}_{\mathcal{B}^{\tau}_{j}}({\mathcal{W}_{L}}_{s, j}^{(n, \tau)})\gets$Sending back to client $j$\\
    \tcp*[h]{Historical Momentum Update}\\
    \If{$\tau = \lvert T_{j} \lvert$}
        {
        \colorbox{pink}{$\mathcal{H}^{n}$}$\gets$ Recording $m_{s, j}^{(n, \tau+1)}$\\
        }
    }
    \tcp*[h]{Momentum Alignment}\\
    \colorbox{pink}{$\bar{m}_{s}^{(n, \tau+1)}$}$\gets$ Updating by Equation~\ref{eq:approach_5}\\
}
\caption{The Step-wise Momentum Fusion (SMoFi) in the $n$-th round}
\label{algo:pseudocode_1}
\end{algorithm}

The full workflow of split FL integrating \SysName is detailed in Appendix~\ref{subsec:framework_workflow}. The central server\footnote{In SFLV1/SFLV2, a fed server is used to aggregate the client-side submodels for the client-side global model. For simplicity, we let the central server performs the operations of the fed server.} first aggregates both client-side and server-side submodels at the end of the communication round: $\bar{\mathcal{W}}^{n}=\sum_{j\in\mathcal{J}^{n}}p_{j}\mathcal{W}_{j}^{n}$ where $\mathcal{W}_{j}^{n}=[\mathcal{W}_{c, j}^{(n, T_{j})}, \mathcal{W}_{s, j}^{(n, T_{j})}]$. Similar to work~\citep{hsu2019measuring}, we update the global model with momentum: $\mathcal{W}^{n} = \mathcal{W}^{n-1}-m_{g}^{n}$. The global momentum buffer is updated following $m_{g}^{n} = \beta_{g}m_{g}^{n-1} + \mathcal{W}^{n-1} - \bar{\mathcal{W}}^{n}$ and $\beta_{g}$ is the global momentum coefficient.

\section{Convergence Analysis}
\label{sec:convergence_analysis}
In this section, we provide the convergence analysis of \SysName under the practical partial client participation. Detailed analysis and complete proofs are provided in Appendix~\ref{appendix_convergence}. We start by stating assumptions commonly adopted in prior work~\citep{acarfederated, rodio2023federated}.
\begin{assumption}
($L$-Smooth Objectives)
The local objective $\mathcal{L}_j = \mathcal{L}_{\mathcal{D}_{j}}, \forall j \in \mathcal{J}$ is L-smooth ($L > 0$), i.e., $\forall \mathcal{W}, \mathcal{W}^{\prime}$, it satisfies
\begin{equation}
\mathcal{L}_{j}(\mathcal{W}) 
\leq 
\mathcal{L}_{j}(\mathcal{W}^{\prime}) + 
\langle \nabla\mathcal{L}_{j}(\mathcal{W}^{\prime}), \mathcal{W} - \mathcal{W}^{\prime} \rangle +
\frac{L}{2}\Vert\mathcal{W}-\mathcal{W}^{\prime}\Vert_{2}^{2}.
\label{eq:main_convergence_1}   
\end{equation}
\label{main_assumption_1}
\end{assumption}

\begin{assumption}
($\mu$-Strongly Convex Objectives)
The local objectives $\mathcal{L}_{\mathcal{D}_{1}}, \cdots, \mathcal{L}_{\mathcal{D}_{\lvert \mathcal{J} \lvert}}$ are all convex, i.e., $\forall \mathcal{W}, \mathcal{W}^{\prime}$, it satisfies
\begin{equation}
\mathcal{L}_{j}(\mathcal{W}) 
\geq 
\mathcal{L}_{j}(\mathcal{W}^{\prime}) + 
\langle \nabla\mathcal{L}_{j}(\mathcal{W}^{\prime}), \mathcal{W} - \mathcal{W}^{\prime} \rangle +
\frac{\mu}{2}\Vert\mathcal{W}-\mathcal{W}^{\prime}\Vert_{2}^{2}.
\label{eq:main_convergence_2}   
\end{equation}
\label{main_assumption_2}
\end{assumption}

\begin{assumption}
(Unbiased Gradient and Bounded Variance)
For mini-batch $\mathcal{B}^{\tau}_{j}$ uniformly sampled at random from local data of $j$-th client $\mathcal{D}_{j}$, the resulting stochastic gradient is unbiased to the gradient entire local dataset, that is, $\mathbb{E}_{\mathcal{B}^{\tau}_{j}\sim\mathcal{D}_{j}}[\nabla\mathcal{L}_{\mathcal{B}^{\tau}_{j}}(\mathcal{W})] = \nabla\mathcal{L}_{j}(\mathcal{W})$. Also, the variance of the stochastic gradient is bounded, i.e., $\forall \tau, j\in\mathcal{J}$, there exists $\sigma$ satisfying
\begin{equation}
\mathbb{E}_{\mathcal{B}^{\tau}_{j}\sim\mathcal{D}_{j}}[\Vert \nabla\mathcal{L}_{\mathcal{B}^{\tau}_{j}}(\mathcal{W}) - \nabla\mathcal{L}_{j}(\mathcal{W}) \Vert_{2}^{2}]
\leq
\sigma^{2}.
\label{eq:main_convergence_3}
\end{equation}
\label{main_assumption_3}
\end{assumption}

\begin{assumption}
(Bounded Gradients) The stochastic gradient is bounded; i.e., $\forall \tau, j\in\mathcal{J}$ there exists $G$ satisfying $\mathbb{E}_{\mathcal{B}^{\tau}_{j}\sim\mathcal{D}_{j}}[\Vert \nabla\mathcal{L}_{\mathcal{B}^{\tau}_{j}}(\mathcal{W}) \Vert_{2}^{2}] \leq G^{2}$.
\label{main_assumption_4}
\end{assumption}

The work~\citep{han2024convergence} offers convergence guarantees for both SFLV1 and SFLV2. The convergence analysis of \SysName follows the idea of this work, as \SysName introduces modifications to the SFLV1 workflow. Based on Assumptions~\ref{main_assumption_1}-~\ref{main_assumption_4}, we provide the convergence guarantee for \SysName in Theorem~\ref{main_theorem_1}.

\begin{theorem}
Under the Assumptions~\ref{main_assumption_1}, \ref{main_assumption_2}, \ref{main_assumption_3}, and \ref{main_assumption_4}, \SysName has the similar convergence guarantees with SFLV1 with the momentum SGD as the optimization solver. Given the predefined communication rounds $N$, client participation rate $\theta$, and a small enough learning rate $\eta^{n}=\frac{4}{\mu(\gamma+n)}$, the error between the global model at $N$-th round and the global optimum is bounded by
%However, the \SysName has a tighter error bound in Equation~\ref{eq:appendix_2_11}:
\begin{equation}
    \scriptsize
    \mathbb{E}[\mathcal{L} ({\mathcal{W}^{N})}] - \mathcal{L}{({\mathcal{W}^{*})}} \leq 
    \mathcal{O}(\frac{A} {(\gamma+N)})
    +
    \mathcal{O}(\frac{B} {(\gamma+N)})
    +
    \mathcal{O}(\frac{C} {{(\gamma+N)}}).
\end{equation}

The $A$, $B$, $C$, and $\gamma$ in the error bound follows $A = \lvert\mathcal{J}\lvert \sum_{j\in\mathcal{J}}p_{j}^{2}(2\sigma^{2}+(1+\frac{1}{\theta})G^{2})$, $B = \sum_{j\in\mathcal{J}}p_{j}(2\sigma^{2}+G^{2})$, $C = \Vert \mathcal{W}^{0} - \mathcal{W}^{*} \Vert$, and $\gamma = 8L/\mu-1$.
\label{main_theorem_1}
\end{theorem}

It indicates that the convergence bound of \SysName achieves an order of $\mathcal{O}(1/N)$.
\section{Experimental Results}
\label{sec:exp}

\begin{table*}[t]
\vspace{-3mm}
\caption{Performance comparison between \SysName and momentum-based counterparts across \textit{three baseline methods} and \textit{three benchmark datasets}. Methods denoted with \textbf{+} represent baselines \textbf{combined} with \SysName or its counterparts. We report the average and standard deviation of \textit{Top-1 accuracy}, the number of communication rounds (\textbf{R}) required to reach the target accuracy (i.e., 90\% of the best global model accuracy by FedAvg), and the corresponding convergence speedup (\textbf{R}$\uparrow$). All results are averaged over three trials, with \textbf{bold} font indicating the best performance for each setup.}
\resizebox{\linewidth}{!}{
\begin{tabular}{lcccccccccc}
\toprule[1pt]
\textbf{Setup} && \multicolumn{2}{c}{\textbf{CIFAR-10/DIR{\scriptsize\textbf{100}}\textbf{(0.2)}}}    && \multicolumn{2}{c}{\textbf{CIFAR-100/DIR{\scriptsize\textbf{100}}\textbf{(0.2)}}}        && \multicolumn{2}{c}{\textbf{Tiny-ImageNet/DIR{\scriptsize\textbf{200}}\textbf{(0.2)}}}\\ 
\cmidrule(r){1-1} \cmidrule(r){3-4} \cmidrule(r){6-7} \cmidrule(r){9-10}
\textbf{Methods}  && \textbf{Acc. (\%)} & \textbf{R}/\textbf{R}$\uparrow$ && \textbf{Acc. (\%)} & \textbf{R}/\textbf{R}$\uparrow$ && \textbf{Acc. (\%)} & \textbf{R}/\textbf{R}$\uparrow$ \\ \hline 
\textit{FedAvg}~\citep{mcmahan2017}  && 77.16\scriptsize$\pm$0.11 & 258/1.00$\times$  && 48.10\scriptsize $\pm$0.36 & 183/1.00$\times$  && 33.43\scriptsize$\pm$0.12 & 161/1.00$\times$ \\
\textbf{+} FedAvgM~\citep{hsu2019measuring} && 79.19\scriptsize$\pm$0.09 & 190/1.36$\times$  && 50.28\scriptsize $\pm$0.26 & 126/1.45$\times$  && 33.58\scriptsize$\pm$0.34 & 57/2.82$\times$  \\
\textbf{+} SlowMo~\citep{wang2019slowmo} && 76.54\scriptsize$\pm$0.06 & 177/1.46$\times$  && 50.96\scriptsize $\pm$0.23 & 125/1.46$\times$  && 33.82\scriptsize$\pm$0.29 & 44/3.66$\times$ \\
\textbf{+} FedNAG~\citep{yang2022federated} && 78.24\scriptsize$\pm$0.43 & 170/1.52$\times$  && 48.30\scriptsize $\pm$1.06 & 198/0.92$\times$  && 30.94\scriptsize$\pm$0.44 & 335/0.48$\times$ \\
\rowcolor{gray!20} \textbf{+} \SysName && \textbf{81.82\scriptsize$\pm$0.61} & \textbf{56/4.61$\times$}  && \textbf{53.83\scriptsize $\pm$0.79} & \textbf{64/2.86$\times$}  && \textbf{39.73\scriptsize$\pm$0.05} & \textbf{16/10.06$\times$}  \\ \hline

\textit{FedProx}~\citep{li2020federated}  && 77.38\scriptsize $\pm$0.01 & 167/1.00$\times$  && 48.67\scriptsize $\pm$0.06 & 175/1.00$\times$  && 34.86\scriptsize $\pm$0.89 & 120/1.00$\times$ \\
\textbf{+} FedAvgM~\citep{hsu2019measuring} && 79.26\scriptsize $\pm$0.65 & 207/0.81$\times$  && 50.45\scriptsize $\pm$0.13 & 111/1.58$\times$  && 34.25\scriptsize $\pm$0.20 & 50/2.40$\times$ \\
\textbf{+} SlowMo~\citep{wang2019slowmo} && 76.63\scriptsize $\pm$0.09 & 210/0.80$\times$  && 51.44\scriptsize $\pm$0.65 & 122/1.43$\times$  && 33.67\scriptsize $\pm$0.22 & 85/1.41$\times$ \\
\textbf{+} FedNAG~\citep{yang2022federated} && 77.59\scriptsize $\pm$0.80 & 200/0.84$\times$  && 48.95\scriptsize $\pm$0.04 & 169/1.04$\times$  && 31.20\scriptsize $\pm$0.21 & 284/0.42$\times$ \\
\rowcolor{gray!20} \textbf{+} \SysName && \textbf{81.99\scriptsize $\pm$0.37} & \textbf{54/3.09$\times$}  && \textbf{54.03\scriptsize $\pm$0.31} & \textbf{71/2.46$\times$}  && \textbf{40.79\scriptsize $\pm$0.13} & \textbf{13/9.23$\times$} \\ \hline

\textit{FedNAR}~\citep{li2023fednar}  && 77.21\scriptsize $\pm$0.06 & 255/1.00$\times$  && 48.02\scriptsize $\pm$0.31 & 183/1.00$\times$  && 33.37\scriptsize $\pm$0.34 & 164/1.00$\times$ \\
\textbf{+} FedAvgM~\citep{hsu2019measuring} && 79.21\scriptsize $\pm$0.47 & 190/1.34$\times$  && 50.80\scriptsize $\pm$0.23 & 120/1.53$\times$  && 33.63\scriptsize $\pm$0.76 & 105/1.56$\times$ \\
\textbf{+} SlowMo~\citep{wang2019slowmo} && 76.71\scriptsize $\pm$0.20 & 199/1.28$\times$  && 51.94\scriptsize $\pm$0.18 & 123/1.49$\times$  && 33.57\scriptsize $\pm$0.48 & 161/1.02$\times$ \\
\textbf{+} FedNAG~\citep{yang2022federated} && 77.94\scriptsize $\pm$0.53 & 199/1.28$\times$  && 48.51\scriptsize $\pm$0.24 & 175/1.05$\times$  && 34.17\scriptsize $\pm$0.27 & 63/2.60$\times$ \\
\rowcolor{gray!20} \textbf{+} \SysName && \textbf{81.65\scriptsize $\pm$0.65} & \textbf{46/5.54$\times$}  && \textbf{53.72\scriptsize $\pm$0.42} & \textbf{67/2.73$\times$}  && \textbf{40.47\scriptsize $\pm$0.11} & \textbf{16/10.25$\times$}
\\[-0.65ex]
\bottomrule[1pt]
\end{tabular}
}
\vspace{-3mm}
\label{tab:exp_result_momentum}
\end{table*}

\begin{table*}[t]
\caption{Performance comparison between \SysName and split FL methods across three benchmark datasets. We report the average and standard deviation of \textit{Top-1 accuracy}, along with the number of communication rounds (\textbf{R}) and wall-clock time (\textbf{T}) required to reach the target accuracy (i.e., 90\% of the best global model accuracy by FedAvg).}
\resizebox{\linewidth}{!}{
\begin{tabular}{lccccccccccccc}
\toprule[1pt]
\textbf{Setup} && \multicolumn{3}{c}{\textbf{CIFAR-10/DIR{\scriptsize\textbf{100}}\textbf{(0.2)}}}    && \multicolumn{3}{c}{\textbf{CIFAR-100/DIR{\scriptsize\textbf{100}}\textbf{(0.2)}}} && \multicolumn{3}{c}{\textbf{Tiny-ImageNet/DIR{\scriptsize\textbf{200}}\textbf{(0.2)}}}\\ 
\cmidrule(r){1-1} \cmidrule(r){3-5} \cmidrule(r){7-9} \cmidrule(r){11-13}
\textbf{Methods}  && \textbf{Acc. (\%)} & \textbf{R} & \textbf{T (h)} && \textbf{Acc. (\%)} & \textbf{R} & \textbf{T (h)} && \textbf{Acc. (\%)} & \textbf{R} & \textbf{T (h)} \\ \hline 
SFLV1 ($\bar{\tau}=1$)~\citep{thapa2022splitfed}  && 68.10\scriptsize$\pm$0.57 & $>$1000 & $>$551.62 && 38.43\scriptsize $\pm$0.06 & $>$600 & $>$172.66 && 21.81\scriptsize$\pm$0.98 & $>$400 & $>$578.95\\
SFLV1 ($\bar{\tau}=E$)~\citep{thapa2022splitfed}  && 77.84\scriptsize$\pm$0.17 & 69 & 28.62 && 46.68\scriptsize $\pm$0.21 & \textbf{40} & \textbf{10.23} && 35.47\scriptsize$\pm$0.12 & 44 & 58.57\\
SFLV2~\citep{thapa2022splitfed} && 79.42\scriptsize$\pm$0.04 & 278 & 144.50 && 53.64\scriptsize $\pm$0.51 & 143 & 42.58 && 34.72\scriptsize$\pm$0.95 & 310 & 527.48\\
MergeSFL~\citep{liao2024mergesfl}  && 79.47\scriptsize$\pm$0.09 & 76 & \textbf{15.84} && 50.16\scriptsize $\pm$0.20 & 53 & 11.22 && 34.74\scriptsize$\pm$0.55 & 118 & 152.25\\
\rowcolor{gray!20} \SysName && \textbf{81.82\scriptsize$\pm$0.61} & \textbf{56} & 29.02 && \textbf{53.83\scriptsize $\pm$0.79} & 64 & 18.46 && \textbf{39.73\scriptsize$\pm$0.05} & \textbf{16} & \textbf{23.02}
\\[-0.65ex]
\bottomrule[1pt]
\end{tabular}
}
\label{tab:exp_result_SFL}
\end{table*}

\subsection{Experimental Setups}
\label{subsec:exp_setup}

\begin{figure*}[t]
    \centering
    \includegraphics[width=\linewidth]{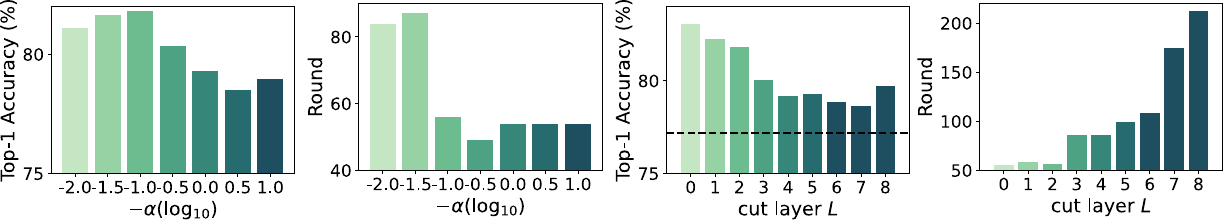}
    \vspace{-5mm}
    \caption{Sensitivity study of \SysName under CIFAR10: (left two) accuracy and convergence under varying staleness factor $\alpha$; (right two) performance under different cut layers $L$. For instance, $L=0$ indicates that all 8 residual blocks and the output block are allocated to the server, while the clients hold only the input block. The dashed line represents the accuracy of FedAvg under the same setting.}
    \vspace{-3mm}
    \label{fig:exp_sensitivity_study}
\end{figure*}

\begin{figure*}[t]
    \centering
    \includegraphics[width=\linewidth]{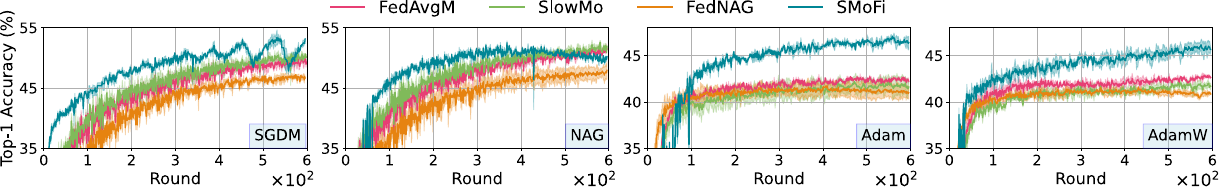}
    \vspace{-6mm}
    \caption{Learning curves of \textit{FedAvg} integrated with \SysName and its counterparts on the CIFAR100 using ResNet-18 under $\mathrm{Dir}_{100}(0.2)$ distribution. From left to right, we investigate different optimizers: SGD with momentum (SGDM), Nesterov Accelerated Gradient (NAG), Adaptive Moment Estimation (Adam), and Adam with decoupled weight decay (AdamW).}
    \vspace{-2mm}
    \label{fig:exp_various_optimizers}
\end{figure*}

\noindent{\textbf{Datasets and Models.}} 
We experiment with three widely used image benchmarks: CIFAR10, CIFAR100~\citep{krizhevsky2009learning}, and Tiny-ImageNet~\citep{le2015tiny}. Additionally, Appendix~\ref{appendix_sub_benchmark} reports evaluation on a language dataset, Shakespeare~\citep{caldas2018leaf} and a speech recognition dataset, Google Speech~\citep{warden2018speech}. We implement the commonly used task model for each dataset: ResNet-18~\citep{he2016deep} for CIFAR10 and CIFAR100; ResNet-34 for Tiny-ImageNet; a stacked transformer model~\citep{vaswani2017attention} for Shakespeare; and VGG-11~\citep{simonyan2014very} for Google Speech. To further validate the robustness of \SysName in different task models, we also explore various models including VGG, MobileNetV2~\citep{sandler2018mobilenetv2}, and DenseNet~\citep{huang2017densely}. In the $n$-th communication round, the server randomly selects 20\% clients $\mathcal{J}^{n}\subseteq\mathcal{J}$ for participation. 

\noindent{\textbf{Heterogeneous Clients Setup.}} 
We simulate \textit{data heterogeneity} in line with previous work \citep{hsu2019measuring, li2022federated}, where the $j$-th client possesses data in the distribution $\boldsymbol{q_{j}} \in \mathbb{R}^{c}$ ($c$ is the number of classes). We sample $\boldsymbol{q_{j}}$ from a Dirichlet distribution $\mathrm{Dir}_{\lvert\mathcal{J}\lvert}(\gamma)$ with a tunable concentration parameter $\gamma>0$ that controls the level of heterogeneity. A smaller $\gamma$ indicates a more heterogeneous distribution setting. Specifically, we set $\mathrm{Dir}_{100}(0.2)$ for CIFAR10 and CIFAR100, and $\mathrm{Dir}_{200}(0.2)$ for Tiny-ImageNet. We also simulate the \textit{system heterogeneity} by varying computing power and communication bandwidth across clients, to assess the wall-clock time efficiency of various split FL frameworks, as detailed in Appendix~\ref{appendix_sub_system_heterogeneity}.

\noindent{\textbf{Baselines.}} 
The baselines chosen for comparison are categorized into momentum-based methods and split FL methods. We compare \SysName against momentum-based counterparts on three baselines: 1) the vanilla FL framework \textit{FedAvg} \citep{mcmahan2017}; 2) \textit{FedProx}~\citep{li2020federated}, adding a proximal term into the local objective function; and 3) \textit{FedNAR}~\citep{li2023fednar} with self-adjusted weight decay. Building upon each baseline, we integrate three momentum-based methods including: 1) \textit{FedAvgM}~\citep{hsu2019measuring}, applying server momentum during global updates; 2) \textit{SlowMo}~\citep{wang2019slowmo}, periodically synchronizing and updating the local momentum across clients; and 3) \textit{FedNAG}~\citep{yang2022federated}, implementing Nesterov Accelerated Gradient (NAG)~\citep{sutskever2013importance, bengio2013advances} for local training with periodic momentum aggregation. Additionally, we include three split FL frameworks for evaluation: SFLV1, SFLV2~\citep{thapa2022splitfed}, and MergeSFL~\citep{liao2024mergesfl}. For SFLV1, we investigate two server-side aggregation frequencies: after every training step ($\bar{\tau}=1$) and after each local epoch ($\bar{\tau}=E$). Note that the performance of SFLV1 is equivalent to FedAvg when aggregating surrogate server-side models at each communication round.

For \SysName, the staleness factor $\alpha$ is fixed at -0.1 across all settings, while $\beta_g$ varies by task: 0.3 for CIFAR10, and 0.5 for CIFAR100 and Tiny-ImageNet. To ensure a fair comparison, we also fine-tune the hyperparameters for all baselines and counterpart methods, with further details provided in Appendix~\ref{appendix_sub_hyperparameters}.

\noindent{\textbf{Metrics.}} 
We run all the methods under each setup three times and report the average \emph{Top-1 accuracy} within 1000, 600, and 400 communication rounds for CIFAR10, CIFAR100, and Tiny-ImageNet, respectively. To evaluate the convergence speed, we calculate the \emph{round-to-accuracy} (\textbf{R}) performance, defined as the number of communication rounds required for the global model to reach the target accuracy--90\% of the best performance achieved by FedAvg--across all settings. We also report the \emph{time-to-accuracy} performance (\textbf{T}) from a system efficiency perspective when evaluating the split FL frameworks.

\subsection{Performance Evaluation}
\label{subsec:exp_eva}

\noindent{\textbf{\SysName Effectiveness.}} \SysName can be easily integrated into other split FL frameworks as a plug-in approach. \SysName improves baseline performance by: 1) speeding up the global model convergence, thereby reducing the overall latency, and 2) further improving the global model performance. 

Table~\ref{tab:exp_result_momentum} compares original baselines (i.e., FedAvg, FedProx, and FedNAR) and the baselines combined with \SysName and momentum-based counterpart methods (denoted with \textbf{+}) on three datasets. In each setup, the full model in \SysName is split at the shallow layers, with the server-side model holding the majority of the task model and the client-side model restricted to the bottom few layers. For instance, in the CIFAR10 task with ResNet-18, we fix the cut layer at $L=2$ for all participating clients, where each client trains a small portion of the model comprising the input block and two residual blocks, while the server-side model includes the remaining 6 residual blocks and the output block. Such a model splitting strategy aligns with practical SFL deployment, where the server typically has significantly greater computational resources than edge devices (i.e., clients), allowing more training tasks to be allocated to the central server for better training efficiency gains.

Experimental results show that \SysName consistently improves the performance of baselines across all benchmarks. The improvements are bigger in the complex classification tasks, particularly in complex tasks such as Tiny-ImageNet (200 classes) using the ResNet-34 model, making it highly suitable for scenarios favoring split training over conventional FL~\citep{singh2019detailed}. Compared to the three momentum-based counterparts, \SysName guarantees objective consistency in higher frequency by step-wise momentum alignment during model training, thereby further improving the accuracy of the global model. Moreover, \SysName speeds up the convergence of the global model to the target accuracy by a large margin compared to baselines integrated with the counterpart methods. Unlike FedProx, which requires clients to report the local weight information--in addition to the activations of the cut layer--to the server when applied in a split training framework, \SysName maintains the same level of privacy guarantee as SplitFed~\citep{thapa2022splitfed} in terms of its client-side transparency without extra data reporting.

Table~\ref{tab:exp_result_SFL} shows the results for split FL frameworks with a constant cut layer at $L=2$ for fair comparison. Performance for SFLV1 is sensitive to the aggregation frequency at the server side: step-wise aggregation underperforms, failing to achieve the target accuracy within the given communication rounds, whereas epoch-wise aggregation yields better performance. SFLV2 outperforms SFLV1 in most cases, consistent with findings in work~\citep{han2024convergence}, albeit at the cost of increased latency due to the sequential interaction between the server and clients. The MergeSFL offers fast convergence speed in terms of temporal space, even though it requires more communication rounds compared with our \SysName, due to the adaptive batch size depending on device capabilities. However, \SysName consistently provides superior performance in the long run.

\begin{figure}[t]
    \centering
    \includegraphics[width=\linewidth]{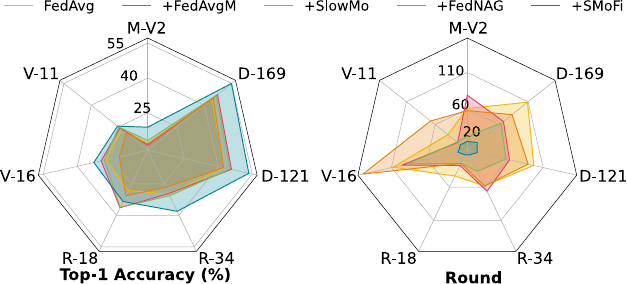}
    \vspace{-5mm}
    \caption{Robustness analysis on the Tiny-ImageNet dataset under $\mathrm{Dir}_{200}(0.2)$ distribution with various task-specific models: VGG (V), MobileNet (M), ResNet (R), and DenseNet (D). We report the best global model performance (left) and round-to-accuracy performance (right), within 150 communication rounds.}
    \label{fig:exp_various_models}
    \vspace{-7mm}
\end{figure}

\noindent{\textbf{Sensitivity Study.}} In Figure~\ref{fig:exp_sensitivity_study}, we investigate the sensitivity of \SysName to the staleness factor $\alpha$ (ranging from $-0.01$ to $-10$) and the cut layer $L$ across the 8 residual blocks of ResNet-18. From the results, a smaller $\alpha$, such as $\alpha=-1$, assigns lower weights to historical momentums during alignment, allowing the global model to converge faster, but can lead to suboptimal model performance. To trade off the model accuracy and convergence, we take $\alpha=-0.1$ as the default setting for all experiments on \SysName. Moreover, performance gains from \SysName are consistent across all cut layers $L$, with more significant benefits when the model is split at shallower layers (i.e., a smaller $L$). This aligns with typical split FL deployments, where a powerful server holds the majority of the model training task. Further sensitivity analysis on additional benchmarks and the ablation study are provided in Appendix~\ref{appendix_sub_sensitivity} and~\ref{appendix_sub_ablation_study}, respectively.

\noindent{\textbf{Robustness Analysis.}} We evaluate the robustness of momentum alignment in \SysName from two perspectives: performance across different optimizers and model architectures. In Figure~\ref{fig:exp_various_optimizers}, \SysName not only converges faster and yields higher accuracy than its counterparts under the SGDM optimizer (i.e., as used in Table~\ref{tab:exp_result_momentum} and Table~\ref{tab:exp_result_SFL}), but also consistently outperforms them when using NAG~\citep{sutskever2013importance}, Adam~\citep{kingma2014adam}, and AdamW~\citep{loshchilov2017decoupled}. Note that, for local optimizers like Adam or AdamW, we periodically align both the first- and second-moment estimates on the server side. Figure~\ref{fig:exp_various_models} further investigates the robustness of \SysName on Tiny-ImageNet across various model architectures. We report both the best global accuracy within 150 communication rounds and the round-to-accuracy performance, where the target accuracy is defined as the performance of FedAvg using the corresponding task model. Results show that \SysName consistently improves both accuracy and convergence across all 7 types of models. The benefits of momentum alignment are more obvious for deeper or more complex model architectures, making \SysName particularly suitable for split training scenarios involving resource-constrained clients and a powerful central server.
\section{Related Work}
\label{sec_related_work}

\noindent{\textbf{Split Federated Learning.} The concept of split~learning was first introduced in works~\citep{gupta2018distributed, vepakomma2018split} to split neural layers into two parts and assign them to the devices (with data resources) and the server (with supercomputing resources). SplitFed~\citep{thapa2022splitfed} takes this a step further by integrating it into the FL framework %to train the models across multiple devices in parallel. 
Current research in split FL primarily aims to address two key challenges: reducing training latency and mitigating the risk of privacy leakage. CPSL~\citep{wu2023split} first partitions devices into several clusters. Training across the clusters follows the same sequential way as SL, while training devices within the cluster in parallel; FedGKT~\citep{he2020group} deploys a compact CNN (composing a lightweight feature extractor and a classifier) on device and the majority of the large model on the server; Works ~\citep{vepakomma2019reducing, abuadbba2020can, pasquini2021unleashing} focus on reducing the privacy leakage in split learning and defending against adversarial attacks.

\noindent{\textbf{Data Heterogeneity.}} Existing methods for addressing the challenges of non-IID data silos can be broadly categorized into three types: 1) loss function modification such as works~\citep{li2021fedrs, gao2022feddc, li2020federated, li2021model} reducing the inconsistency across clients by adding the penalty term into the local objective; 2) robustness aggregation by re-weighting the local updates~\citep{wang2020tackling}, alternatively, taking aggregation as optimization problem and applying various optimizer%to average the updates
~\citep{reddi2020adaptive}; 3) adaptive hyperparameter setting, for instance, learning rate and weight decay of each local SGD solver~\citep{li2023fednar}, or selection rate in each communication round~\citep{balakrishnan2022diverse}. Among these approaches, works~\citep{hsu2019measuring, wang2019slowmo, yang2022federated} apply momentum-based updating in either local training or central aggregation.

%The aforementioned methods for addressing non-IID issues are primarily designed for FL; however, directly applying these methods to the split FL context may be sub-optimal, as discussed in Section~\ref{sec:intro}. Our work leverages the deep engagement of the server with the training process to impose step-wise constraints on model training, allowing for finer granularity of control of model learning in split FL. %, which is difficult to achieve in FL due to the restricted involvement of the central server in the training process.

\vspace{-2mm}
\section{Conclusion}
\label{sec_conclusion}
\vspace{-1mm}

In this paper, we revisit the use of momentum to improve the performance of split FL on non-IID data silos. We propose \SysName, a simple yet effective split FL framework that aligns the momentum of server-side solvers at each learning step. By leveraging the inherent client-server interaction in split FL, \SysName imposes gradient-based constraints to mitigate training divergence. Experimental results show that \SysName significantly improves both convergence speed and accuracy of the global model. Moreover, \SysName requires zero modifications on clients, making it fully client-transparent--without additional communication overhead or privacy risk--thus offering a practical solution for real-world deployment.

\section{Acknowledgments}
This work is partly funded by the EU's Horizon Europe HarmonicAI project under the HORIZON-MSCA-2022-SE-01 scheme with grant agreement number 101131117. This work is also supported by the ICAI GENIUS lab of the research program ROBUST (project number KICH3.LTP.20.006), partly funded by the Dutch Research Council (NWO). We also thank the support provided by Samenwerkende Universitaire RekenFaciliteiten (SURF) with their Snellius infrastructure.

\bibliography{reference}

% \clearpage
% \onecolumn
% \twocolumn
\appendix
% \clearpage
\section{Experiments Details}
\label{appendix_exp_details}

\subsection{Datasets and Models}
\label{appendix_sub_dataset_models}

\begin{figure*}
    \centering
    \includegraphics[width=0.9\linewidth]{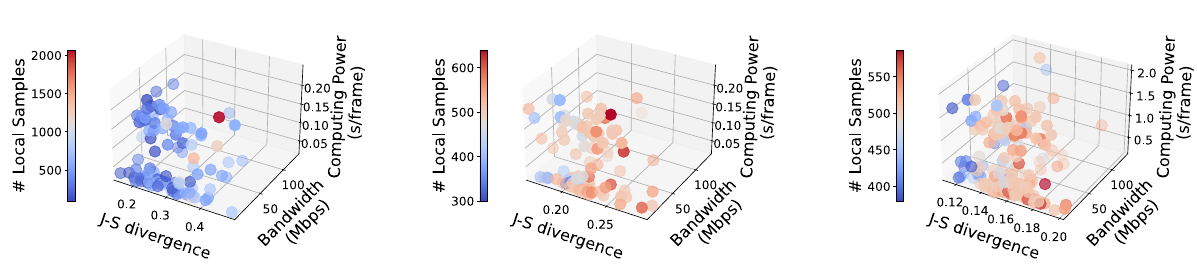}
    \vspace{-4mm}
    \caption{Visualization of the data and system heterogeneity across three benchmarks (from left to right): CIFAR10, CIFAR100, and Tiny-ImageNet. Data heterogeneity is reflected by varying local dataset sizes and distributions (quantified by J-S divergence). The system heterogeneity is simulated by endowing each client with varying levels of computing power--parameterized by model inference speed in \textit{s/frame}--and network throughput.}
    \label{fig:supp_system_heterogeneity_setup}
\end{figure*}

\begin{figure*}
    \centering
    \includegraphics[width=\linewidth]{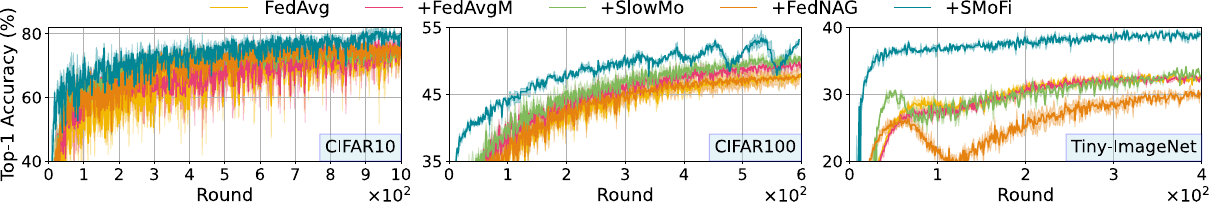}\\
    \vspace{1.5mm}
    \includegraphics[width=\linewidth]{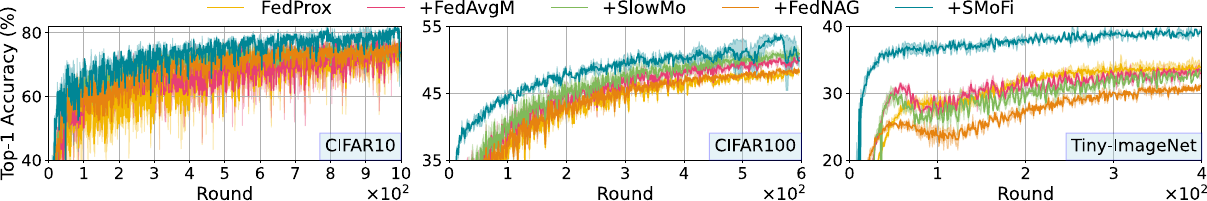}\\
    \vspace{1.5mm}
    \includegraphics[width=\linewidth]{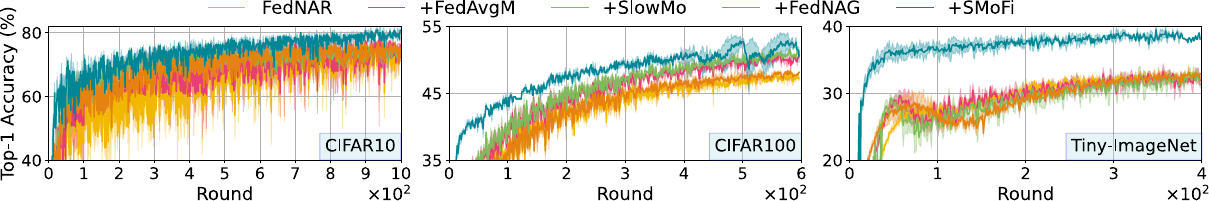}\\
    \vspace{-2mm}
    \caption{Learning curves of \SysName and its counterparts integrated into three baseline methods: FedAvg (top), FedProx (middle), and FedNAR (bottom). Each setup includes experiments on three benchmarks (from left to right): CIFAR-10 and CIFAR-100 with ResNet-18 under $\mathrm{Dir}_{100}(0.2)$, and Tiny-ImageNet with ResNet-34 under $\mathrm{Dir}_{200}(0.2)$.}
    % \vspace{-2mm}
    \label{fig:supp_learning_curve_baselines}
\end{figure*}

\begin{figure*}
    \centering
    \includegraphics[width=\linewidth]{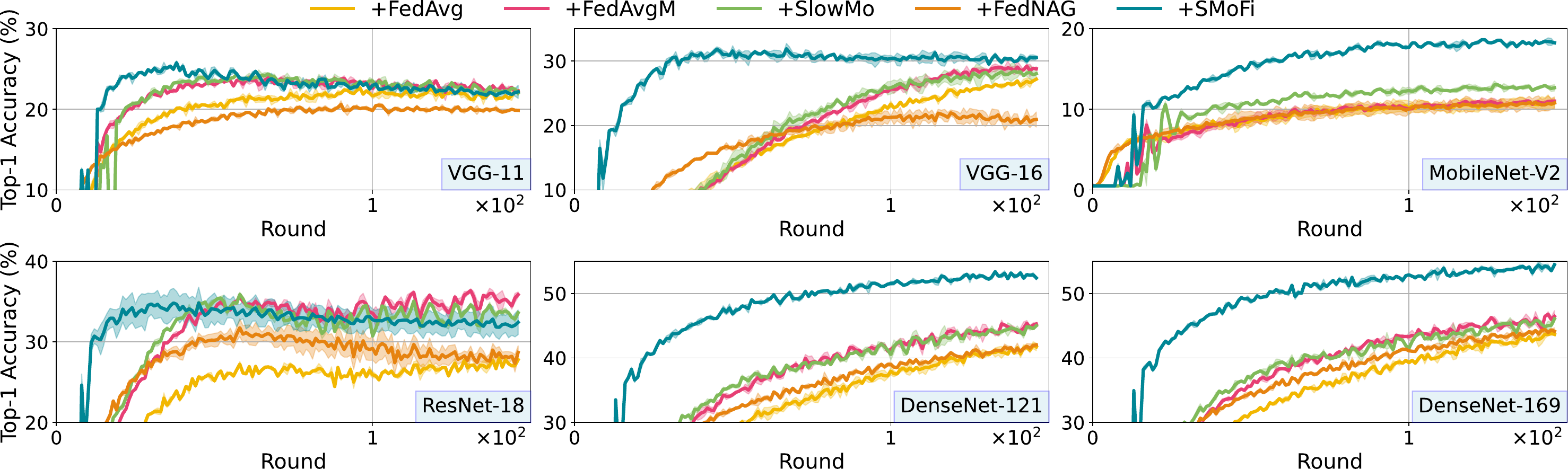}
    \vspace{-8mm}
    \caption{Learning curves of \SysName and its counterparts on Tiny-ImageNet benchmark with different task models. We evaluate the robustness of each momentum-based method when integrated into FedAvg, under the fixed training budget of 150 communication rounds. }
    \label{fig:supp_learning_curve_varying_model}
\end{figure*}

\noindent{\textbf{CIFAR10 and CIFAR100.}} 
Both datasets are used for image classification tasks, with CIFAR-10 containing 10 classes and CIFAR-100 containing 100 classes. We employ ResNet-18 as the task model and use the SGDM optimizer with a momentum of 0.9 and a mini-batch of size $B=32$. The initial learning rate is set to 0.05, decaying across rounds by a factor of 0.998, with a weight decay of 0.0005. For each dataset, we assign the 50000 training samples to 100 clients following a Dirichlet distribution with a concentration parameter of 0.2. Besides, the hyperparameters for optimizers in experiments shown in Figure~\ref{fig:exp_various_optimizers} are set as follows: the NAG optimizer shares the same settings as SGDM, while both Adam and AdamW optimizers use a learning rate of 0.001 without weight decay.

\noindent{\textbf{Tiny-ImageNet.}} It is a more complex image classification dataset with 200 classes. The task model for Tiny-ImageNet is ResNet-34, updated by the SGDM optimizer with a momentum of 0.9 and a mini-batch of size $B=64$. The initial learning rate is set to 0.05, decaying across rounds by a factor of 0.998, with a weight decay of 0.001. Similarly, we assign the 100000 training samples to 200 clients following a Dirichlet distribution with a concentration parameter $\gamma=0.2$. In Figure~\ref{fig:exp_various_models}, we validate the robustness of \SysName to different task models on the Tiny-ImageNet benchmark. The SGDM optimizer with a constant momentum of 0.9 across all task models while the learning rate and weight decay vary as follows: for VGG-11, VGG-16, and MobileNet-V2, the learning rate is 0.01 with a weight decay of 0.0005; for ResNet-18 the learning rate is 0.05 with a weight decay of 0.0005; and for DenseNet-121 and DenseNet-169, the learning rate is 0.05 with a weight decay of 0.001.

\noindent{\textbf{Shakespeare.}} It is a language dataset from the collection of \textit{The Complete Works of William Shakespeare}, which is used for the next-character prediction task with 80 classes. In line with work~\citep{li2023fednar}, we use a stacked transformer model with six attention layers as the task model backbone. The SGDM optimizer is configured with a momentum of 0.9, a learning rate of 0.01, a weight decay of 0.0005, and a mini-batch of size $B=100$. Due to the inherently non-IID nature of the Shakespeare dataset, we randomly assign 100 roles to the corresponding clients.

\noindent{\textbf{Google Speech.}} It is a speech command recognition dataset with 35 classes, consisting of common words such as ``Yes'', ``No'', ``Up'', ``Down'', ``Left'', ``Right'', ``On'', ``Off'', ``Stop'', and ``Go''. We assign a total of 94824 audio clips to 200 clients following a Dirichlet distribution, and randomly select 20 clients for training in each round. The test set contains 11005 audio clips for evaluating the global model. We use VGG-11 as the task model for speech recognition, and configure the SGDM optimizer with a momentum of 0.9, a learning rate of 0.01, a weight decay of 0.0005, and a mini-batch of size $B=32$.

\subsection{Hyperparameter Choice}
\label{appendix_sub_hyperparameters}
We run all the experiments on an NVIDIA A40 GPU. For a fair comparison, we perform the search over hyperparameters to report the best performance of each compared method across all experimental setups.

FedProx adds a proximal term into the local objective function for local training consistency. We select the optimal penalty constant $\mu_{\mathrm{prox}}$ via a grid search over $\{0.0001, 0.001, 0.01, 0.1\}$. The best $\mu_{\mathrm{prox}}$ for CIFAR10, CIFAR100, and Tiny-ImageNet are 0.01, 0.01, and 0.1, respectively.

FedAvgM introduces server momentum during global model updating. The server momentum factor $\beta_{\mathrm{avgm}}$ in FedAvgM varies by tasks: $\beta_{\mathrm{avgm}}=0.3$ for CIFAR10; $\beta_{\mathrm{avgm}}=0.5$ for CIFAR100 and Tiny-ImageNet. Experiments shown in Figure~\ref{fig:exp_various_optimizers} and Figure~\ref{fig:exp_various_models}, $\beta_{\mathrm{avgm}}$ are kept constant at 0.5.

In SlowMo, clients periodically synchronize and perform a momentum update. We fine-tune the slow learning rate $\alpha_{\mathrm{slow}}$ and momentum $\beta_{\mathrm{slow}}$ in SlowMo: $\alpha_{\mathrm{slow}}=0.5$ and $\beta_{\mathrm{slow}}=0.4$ for CIFAR10; $\alpha_{\mathrm{slow}}=1$ and $\beta_{\mathrm{slow}}=0.6$ for CIFAR100, Tiny-ImageNet. For the robustness analysis, $\alpha_{\mathrm{slow}}$ and $\beta_{\mathrm{slow}}$ are kept constant at 1 and 0.6, respectively.

FedNAG uses Nesterov Accelerated Gradient optimizer, following its original design unless otherwise specified.

Besides, we also include three split FL frameworks for evaluation: 1) SFLV1, where the server maintains multiple surrogate server-side models for each client, and updates the server-side models in parallel. The server periodically aggregates the surrogate models at a specified frequency $\bar{\tau}$. We investigate two aggregation frequencies: per local step ($\bar{\tau}=1$) and per local epoch ($\bar{\tau}=E$). When aggregation is performed at the end of each round ($\bar{\tau}=N$), SFLV1 becomes equivalent to FedAvg; 2) SFLV2, where the server maintains a single server-side model and updates the model by sequentially interacts with the clients; and 3) MergeSFL~\citep{liao2024mergesfl} where client-side batch sizes vary according to local computing and communication capabilities. The server updates the server-side model on the mixed activation sequence collected from participating clients at each step.

Moreover, for SlowMo, FedNAG, and MergeSFL, we follow the implementation settings from the original work, where all participating clients execute the same number of local steps in parallel. We specify a fixed number of local steps $T$, based on the mean value of local steps across clients, which varies by dataset: $T=75$ for CIFAR10 and CIFAR100 and $T=40$ for Tiny-ImageNet. For \SysName and other baselines, we maintain that clients execute the same number of local epochs and set it to 5 across all experiments.

\subsection{System Heterogeneity}
\label{appendix_sub_system_heterogeneity}

To investigate the training efficiency of \SysName and its counterpart SFL frameworks, we report the time-to-accuracy performance in Table~\ref{tab:exp_result_SFL}. To this end, we calculate the wall clock time in the context of both data and system heterogeneity. 

As shown in Figure~\ref{fig:supp_system_heterogeneity_setup}, the heterogeneous data silos are visualized by variations in local dataset sizes and distributions (quantified by J-S divergence). Specifically, we quantity the imbalance level of local data $\mathcal{D}_j$ by the J-S divergence between local distribution $\boldsymbol{q_{j}}$ and the balanced distribution $\boldsymbol{\tilde{q_{j}}} = [\lfloor \frac{\lvert \mathcal{D}_j \lvert}{c} \rfloor, \cdots, \lfloor \frac{\lvert \mathcal{D}_j \lvert}{c} \rfloor]$ where $c$ denotes the number of task-related classes.

The heterogeneous resources across clients are reflected in the different inference speed $p^{d}_{j}$ and communication bandwidth $b_{j}$. Specifically, to simulate system heterogeneity in real-world scenarios, client $j\in\mathcal{J}$ is endowed with the computation capability $p^{d}_{j}$ (parameterized by model inference speed in \textit{s/frame}) and network throughput $b_{j}$ (in \textit{kbps}), sampled from the public dataset AI benchmark~\citep{ignatov2019ai} and MobiPerf \citep{huang2011mobiperf}, respectively. The central server is assumed to have significantly greater computational resources, defined as $p^{s}=\frac{1}{\kappa\lvert \mathcal{J} \lvert}\sum_{j \in \mathcal{J}}p^{d}_{j}$, controlled by parameter $\kappa$. This implies that the inference speed at the server side is $\kappa\times$ faster than the average client speed, and we set $\kappa=100$ for all benchmarks. Note that the inference speed varies according to the task-specific models. 

We calculate the accumulated training time \textbf{T} over \textbf{R} rounds--as aligned with Table~\ref{tab:exp_result_momentum}--when the model reaches the target accuracy, defined as 90\% of the best performance achieved by FedAvg in all settings. The training latency for the $j$-th client collaboratively training with the central server over a mini-batch $B$ is
\begin{equation}
    t_{j} = t_{j}^{d-comp} + t_{j}^{s-comp} + t_{j}^{comm},
    \label{eq:supp_exp_1}
\end{equation}
which consists of on-device training time $t_{j,n}^{d-comp}$, on-server training time $t_{j,n}^{s-comp}$, and communication time $t_{j,n}^{comm}$ between the device and server. Each part is defined as:
\begin{equation}
    t_{j}^{d-comp} = 3 \times B \times p^{d}_{j} \times \mathcal{O}(L),
    \label{eq:supp_exp_2}
\end{equation}
\begin{equation}
    t_{j}^{s-comp} = 3 \times B \times p^{s} \times (1-\mathcal{O}(L)),
    \label{eq:supp_exp_3}
\end{equation}
\begin{equation}
    t_{j}^{comm} = 2 \times B \times \mathcal{S}(L) \times \frac{1}{b_{j}},
    \label{eq:supp_exp_4}
\end{equation}
where $B$ is the batch size. The factor $3\times$ in~\eqref{eq:supp_exp_2} and~\eqref{eq:supp_exp_3} follows the assumption that the backward pass is twice as costly as the forward pass (i.e., the model inference)~\cite{lai2022fedscale}. The $\mathcal{O}(\cdot)$ and $\mathcal{S}(\cdot)$ denote the ratio of device-side model operations and the size of transferred activations/gradients. We employ DeepSpeed library to profile the computational complexity (in \textit{MACs}) and output size (in \textit{kb}) of each layer in the task-specific model. Given the cut layer $L$, $\mathcal{O}(L)$ and $\mathcal{S}(L)$ can be retrieved from the model profile.

\section{Additional Experimental Results}
\label{appendix_exp_results}

\subsection{Training Curves}
\label{appendix_sub_curves}

We provide the convergence plots in Figure~\ref{fig:supp_learning_curve_baselines} and Figure~\ref{fig:supp_learning_curve_varying_model}, to complement the experiments in Table~\ref{tab:exp_result_momentum} and Figure~\ref{fig:exp_various_models}.

As shown in Figure~\ref{fig:supp_learning_curve_baselines}, we run all the methods under each setup three times with different random seeds. In this way, the data silos across clients have distinct initializations in each trail for CIFAR10, CIFAR100, and Tiny-ImageNet benchmarks. From the results, we observe consistent improvements in the global model performance by \SysName across three tasks and three baselines (i.e., FedAvg, FedProx, and edNAR). Moreover, the global model in \SysName converges faster than those in the other three counterpart methods (i.e., FedAvgM, SlowMo, and FedNAG).

We also present the learning curve for the robustness analysis in Figure~\ref{fig:supp_learning_curve_varying_model}. We evaluate the performance of \SysName and its counterparts on Tiny-ImageNet with different task model architectures over a maximum of 150 communication rounds. Results indicate that \SysName further enhances the performance and convergence of the global model in the baseline method (i.e., FedAvg) by a large margin, particularly for deep and complex models such as DenseNet-121 and DenseNet-169. Furthermore, the performance gains achieved by \SysName are more robust to the choice of task model compared to the three momentum-based methods.

\subsection{\textbf{Sensitivity Study}}
\label{appendix_sub_sensitivity}

\begin{figure*}
    \centering
    \includegraphics[width=\linewidth]{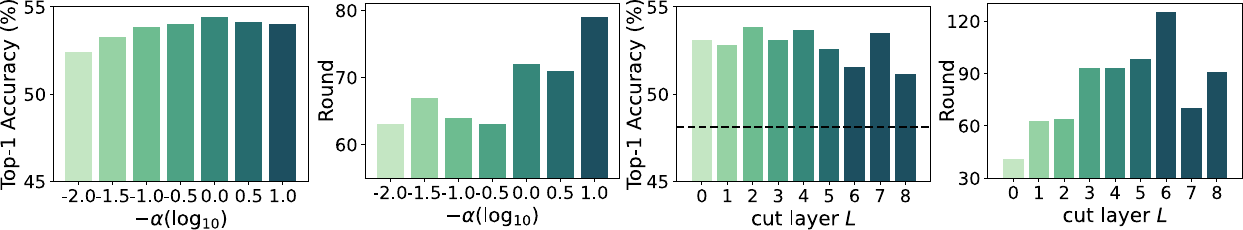}\\
    \vspace{1.5mm}
    \includegraphics[width=\linewidth]{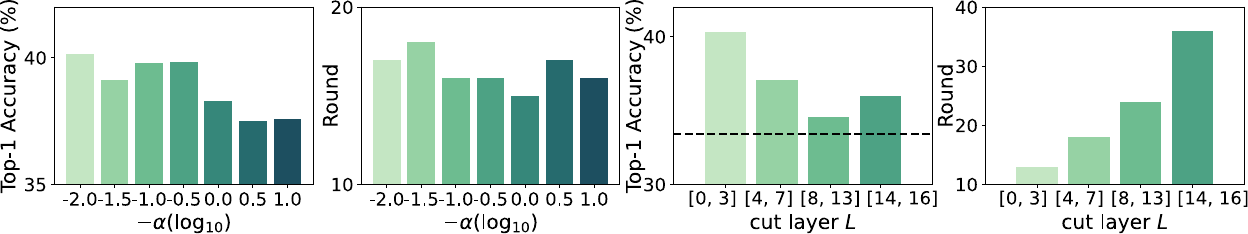}\\
    \vspace{1.5mm}
    \caption{Sensitivity study of \SysName under CIFAR100 (top) and Tiny-ImageNet (bottom) datasets: (left two) accuracy and convergence under varying staleness factor $\alpha$; (right two) performance under different cut layers $L$. For Tiny-ImageNet with ResNet-34, the cut layer in each round is randomly selected from a predefined range. The dashed line represents the accuracy of FedAvg under the same setting for comparison.}
    \vspace{-4mm}
    \label{fig:supp_sensitivity_study}
\end{figure*}

\begin{figure*}
    \centering
    \includegraphics[width=\linewidth]{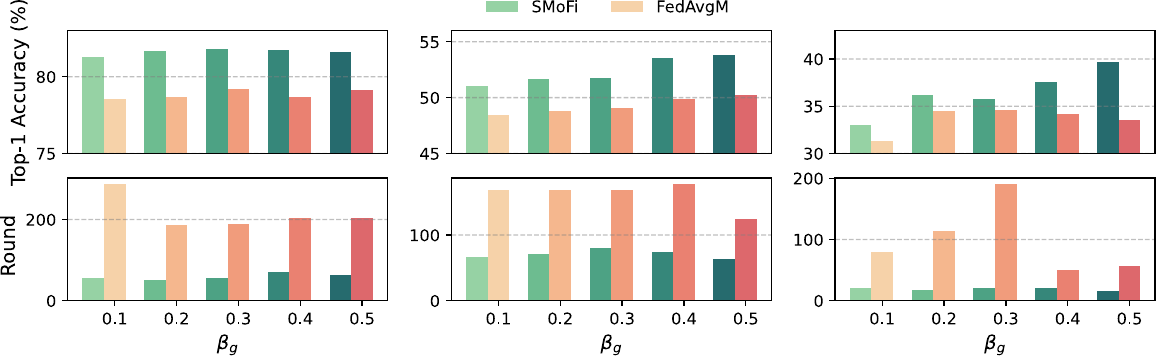}
    \vspace{-6mm}
    \caption{Sensitivity study of \SysName and FedAvgM to the global momentum coefficient $\beta_g$ ranging from 0.1 to 0.5. We report the \textit{Top-1 accuracy} (top) and \textit{round-to-accuracy} performance (bottom) across three benchmarks (from left to right): CIFAR10, CIFAR100, and Tiny-ImageNet. For a fair comparison, we set identical accuracy targets across three datasets: 70\%, 43\%, and 30\%, for CIFAR-10, CIFAR-100, and Tiny-ImageNet, respectively.}
    \label{fig:supp_sensitivity_study_global_momentum}
\end{figure*}

\begin{figure*}
    \centering
    \includegraphics[width=\linewidth]{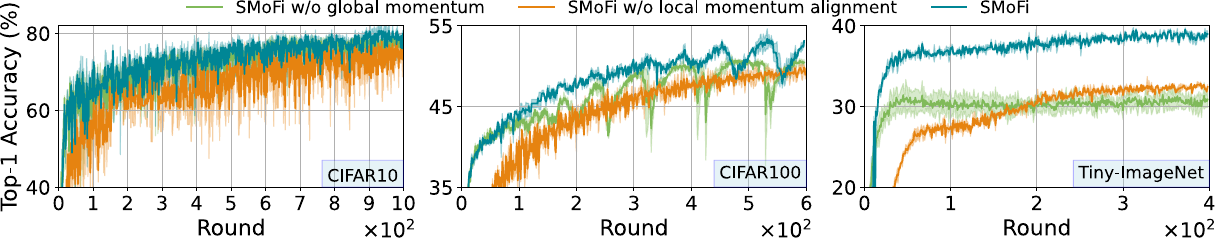}
    \vspace{-6mm}
    \caption{Ablation study for \SysName across three benchmarks (from left to right): CIFAR10, CIFAR100, and Tiny-ImageNet. We investigate the impact of global momentum and local momentum alignment on the performance of \SysName.}
    \label{fig:supp_ablation_study}
\end{figure*}

\begin{table*}[t]
\caption{Performance comparison between \SysName and momentum-based counterparts across three benchmarks under two data distributions: \textbf{moderately non-IID and IID}. Methods denoted with \textbf{+} represent baseline FedAvg combined with \SysName or its counterparts. We report the average and standard deviation of \textit{Top-1 accuracy}, the number of communication rounds (\textbf{R}) required to reach the target accuracy (i.e., 90\% of the best global model accuracy by FedAvg), and the corresponding convergence speedup (\textbf{R}$\uparrow$). All results are averaged over three trials, with \textbf{bold} indicating the best performance for each setup.}
\resizebox{\linewidth}{!}{
\begin{tabular}{lcccccccccc}
\toprule[1pt]
\textbf{Setup} && \multicolumn{2}{c}{\textbf{CIFAR-10/DIR{\scriptsize\textbf{100}}\textbf{(0.5)}}}    && \multicolumn{2}{c}{\textbf{CIFAR-100/DIR{\scriptsize\textbf{100}}\textbf{(0.5)}}}        && \multicolumn{2}{c}{\textbf{Tiny-ImageNet/DIR{\scriptsize\textbf{200}}\textbf{(0.5)}}}\\ 
\cmidrule(r){1-1} \cmidrule(r){3-4} \cmidrule(r){6-7} \cmidrule(r){9-10}
\textbf{Methods}  && \textbf{Acc. (\%)} & \textbf{R}/\textbf{R}$\uparrow$ && \textbf{Acc. (\%)} & \textbf{R}/\textbf{R}$\uparrow$ && \textbf{Acc. (\%)} & \textbf{R}/\textbf{R}$\uparrow$ \\ \hline 
\textit{FedAvg}~\citep{mcmahan2017}  && 82.04\scriptsize$\pm$0.22 & 85/1.00$\times$  && 48.35\scriptsize $\pm$0.58 & 164/1.00$\times$  && 31.57\scriptsize$\pm$2.43 & 126/1.00$\times$ \\
\textbf{+} FedAvgM~\citep{hsu2019measuring} && 82.67\scriptsize$\pm$0.24 & 65/1.31$\times$  && 51.70\scriptsize $\pm$0.67 & 114/1.44$\times$  && 33.87\scriptsize$\pm$1.00 & 100/1.26$\times$  \\
\textbf{+} SlowMo~\citep{wang2019slowmo} && 81.17\scriptsize$\pm$0.30 & 101/0.84$\times$  && 52.78\scriptsize $\pm$0.25 & 100/1.64$\times$  && 35.60\scriptsize$\pm$0.41 & 37/3.41$\times$ \\
\textbf{+} FedNAG~\citep{yang2022federated} && 82.26\scriptsize$\pm$0.14 & 75/1.13$\times$  && 49.98\scriptsize $\pm$0.52 & 147/1.12$\times$  && 31.88\scriptsize$\pm$0.11 & 241/0.52$\times$ \\
\rowcolor{gray!20} \textbf{+} \SysName && \textbf{84.62\scriptsize$\pm$0.07} & \textbf{34/2.50$\times$}  && \textbf{55.54\scriptsize $\pm$0.11} & \textbf{55/2.98$\times$}  && \textbf{40.90\scriptsize$\pm$0.33} & \textbf{12/10.50$\times$}\\

\toprule[1pt]
\textbf{Setup} && \multicolumn{2}{c}{\textbf{CIFAR-10/IID{\scriptsize\textbf{100}}}} && \multicolumn{2}{c}{\textbf{CIFAR-100/IID{\scriptsize\textbf{100}}}} && \multicolumn{2}{c}{\textbf{Tiny-ImageNet/IID{\scriptsize\textbf{200}}}}\\ 
\cmidrule(r){1-1} \cmidrule(r){3-4} \cmidrule(r){6-7} \cmidrule(r){9-10}
\textbf{Methods}  && \textbf{Acc. (\%)} & \textbf{R}/\textbf{R}$\uparrow$ && \textbf{Acc. (\%)} & \textbf{R}/\textbf{R}$\uparrow$ && \textbf{Acc. (\%)} & \textbf{R}/\textbf{R}$\uparrow$ \\ \hline 
\textit{FedAvg}~\citep{mcmahan2017}  && 83.52\scriptsize$\pm$0.20 & 44/1.00$\times$  && 50.29\scriptsize $\pm$0.15 & 164/1.00$\times$  && 32.38\scriptsize$\pm$0.67 & 147/1.00$\times$ \\
\textbf{+} FedAvgM~\citep{hsu2019measuring} && 84.50\scriptsize$\pm$0.11 & 33/1.33$\times$  && 53.47\scriptsize $\pm$0.06 & 109/1.50$\times$  && 35.90\scriptsize$\pm$1.27 & 35/4.20$\times$  \\
\textbf{+} SlowMo~\citep{wang2019slowmo} && 83.24\scriptsize$\pm$0.04 & 55/0.80$\times$  && 54.66\scriptsize $\pm$0.07 & 89/1.84$\times$  && 34.99\scriptsize$\pm$1.68 & 85/1.73$\times$ \\
\textbf{+} FedNAG~\citep{yang2022federated} && 84.36\scriptsize$\pm$0.03 & 42/1.05$\times$  && 50.72\scriptsize $\pm$0.04 & 171/0.96$\times$  && 31.87\scriptsize$\pm$0.44 & 252/0.58$\times$ \\
\rowcolor{gray!20} \textbf{+} \SysName && \textbf{86.88\scriptsize$\pm$0.15} & \textbf{14/3.14$\times$}  && \textbf{55.46\scriptsize $\pm$0.18} & \textbf{60/2.73$\times$}  && \textbf{42.48\scriptsize$\pm$0.32} & \textbf{11/13.36$\times$}

\\[-0.65ex]
\bottomrule[1pt]
\end{tabular}
}
\label{tab:supp_exp_result_DIR_0.5_IID}
\end{table*}

\begin{table*}[t]
\caption{Performance comparison between \SysName and momentum-based counterparts on the Tiny-ImageNet benchmark under \textbf{varying client participation scales}. In each communication round, a subset of clients $\mathcal{J}^{n}$ is randomly selected from 200 clients $\mathcal{J}$, with participation rates of 5\%, 10\%, and 15\%. Methods denoted with \textbf{+} represent baseline FedAvg combined with \SysName or its counterparts. We report the average and standard deviation of \textit{Top-1 accuracy}, the number of communication rounds (\textbf{R}) required to reach the target accuracy (i.e., 90\% of the best global model accuracy by FedAvg), and the corresponding convergence speedup (\textbf{R}$\uparrow$). All results are averaged over three trials, with \textbf{bold} indicating the best performance for each setup.}
\resizebox{\linewidth}{!}{
\begin{tabular}{lcccccccccc}
\toprule[1pt]
\textbf{Setup} && \multicolumn{2}{c}{$\lvert\mathcal{J}^{n}\lvert=10$} && \multicolumn{2}{c}{$\lvert\mathcal{J}^{n}\lvert=20$} && \multicolumn{2}{c}{$\lvert\mathcal{J}^{n}\lvert=30$}\\ 
\cmidrule(r){1-1} \cmidrule(r){3-4} \cmidrule(r){6-7} \cmidrule(r){9-10}
\textbf{Methods}  && \textbf{Acc. (\%)} & \textbf{R}/\textbf{R}$\uparrow$ && \textbf{Acc. (\%)} & \textbf{R}/\textbf{R}$\uparrow$ && \textbf{Acc. (\%)} & \textbf{R}/\textbf{R}$\uparrow$ \\ \hline 
\textit{FedAvg}~\citep{mcmahan2017}  && 31.66\scriptsize$\pm$0.05 & 111/1.00$\times$  && 32.96\scriptsize$\pm$0.62 & 125/1.00$\times$  && 33.15\scriptsize$\pm$0.53 & 139/1.00$\times$ \\
\textbf{+} FedAvgM~\citep{hsu2019measuring} && 33.25\scriptsize$\pm$1.28 & 96/1.16$\times$  && 34.38\scriptsize$\pm$0.84 & 58/2.16$\times$  && 34.89\scriptsize$\pm$0.05 & 53/2.62$\times$ \\
\textbf{+} SlowMo~\citep{wang2019slowmo} && 33.98\scriptsize$\pm$0.19 & 165/0.67$\times$  && 34.91\scriptsize$\pm$0.24 & 64/1.95$\times$  && 34.40\scriptsize$\pm$0.72 & 58/2.40$\times$ \\
\textbf{+} FedNAG~\citep{yang2022federated} && 31.14\scriptsize$\pm$0.18 & 227/0.49$\times$  && 31.43\scriptsize$\pm$1.11 & 244/0.51$\times$  && 31.28\scriptsize$\pm$1.02 & 284/0.49$\times$ \\
\rowcolor{gray!20} \textbf{+} \SysName && \textbf{38.23\scriptsize$\pm$0.64} & \textbf{63/1.76$\times$}  && \textbf{39.58\scriptsize$\pm$0.46} & \textbf{25/5.00$\times$}  && \textbf{39.70\scriptsize$\pm$0.26} & \textbf{20/6.95$\times$}
\\[-0.65ex]
\bottomrule[1pt]
\end{tabular}
}
\label{tab:supp_exp_result_involvement}
\end{table*}

In Figure~\ref{fig:exp_sensitivity_study}, we provide the sensitivity analysis of \SysName on the CIFAR10 dataset. Figure~\ref{fig:supp_sensitivity_study} extends the sensitivity study to the CIFAR100 and Tiny-ImageNet under the same settings: the staleness factor $\alpha$ ranges from -0.01 to -10, and the cut layer $L$ varies across the core blocks of the task models (ResNet-18 for CIFAR100 and ResNet-34 for Tiny-ImageNet). For the staleness factor $\alpha$, performance trends differ between datasets. On CIFAR100, a smaller $\alpha$, which assigns lower weights to historical momentum during alignment, slows down the model convergence while introducing marginal accuracy gains, which signifies the importance of historical momentum to the model convergence on CIFAR100 task. Observations on Tiny-ImageNet dataset are similar to those on CIFAR10: a larger $\alpha$ improves the global model performance in the long run. However, the $\alpha$ shows minimal impact on convergence speed for Tiny-ImageNet.

The effect of varying the cut layer $L$ is constant on CIFAR100 and Tiny-ImageNet, aligning with the observations on CIFAR10: \SysName brings more significant benefits to both accuracy and convergence speed when the model is split at a smaller $L$ layer, leaving a larger portion of the model to the server. For Tiny-ImageNet dataset using ResNet-34, we randomly select the cut layer $L$ in each round from four ranges: $[0,3]$, $[4,7]$, $[8,13]$, and $[14,16]$. Across all configurations, \SysName outperforms FedAvg in terms of accuracy, even in the extreme case where only the output block resides on the server (i.e., $L=8$ in ResNet-18 or $L=16$ in ResNet-34).

Moreover, we investigate the sensitivity of \SysName to the global momentum coefficient $\beta_g$, in comparison with FedAvgM, which introduces this mechanism for global model updates. To ensure a fair comparison of round-to-accuracy performance, we set identical target accuracies for \SysName and FedAvgM across three datasets: 70\% for CIFAR-10, 43\% for CIFAR-100, and 30\% for Tiny-ImageNet. Figure~\ref{fig:supp_sensitivity_study_global_momentum} reports the results under $\beta_g$ ranging from 0.1 to 0.5. We observe that: 1) the global model performance is more sensitive to $\beta_g$ in more complex tasks such as CIFAR100 and Tiny-ImageNet, where a larger $\beta_g$ tends to be more beneficial; 2) \SysName consistently outperforms FedAvgM in both accuracy and convergence owing to its momentum alignment, which imposes tighter, step-wise constraints, in contrast to the round-wise constraints applied only during global model updates in FedAvgM.

\subsection{Ablation Study}
\label{appendix_sub_ablation_study}
To validate the key components in \SysName, we compare \SysName and its variants on three datasets, as shown in Figure~\ref{fig:supp_ablation_study}. The ablation study shows that the step-wise momentum fusion across server-side optimizers significantly benefits global model performance by comparing \textit{\SysName} and \textit{\SysName w/o momentum alignment}. Such performance gain becomes more significant as the task and model complexity increase--from CIFAR10 to CIFAR100 and Tiny-ImageNet. We also observe that the local momentum alignment plays a crucial role in speeding up global model convergence. Besides, applying momentum updates to the global model further improves performance, particularly in tasks such as Tiny-ImageNet with ResNet-34. These two key components of \SysName jointly contribute to faster convergence in early rounds and better performance over the long term.

\subsection{\textbf{More Data Distributions}}
\label{appendix_sub_data_distribution}
In Table~\ref{tab:exp_result_momentum}, we compare the \SysName with momentum-based counterparts under a non-IID setting using a Dirichlet distribution with concentration 0.2, i.e., $\mathrm{Dir}(0.2)$. In this section, we extend our analysis to a moderately heterogeneous setting ($\mathrm{Dir}(0.5)$) and an $\mathrm{IID}$ setup. In Table~\ref{tab:supp_exp_result_DIR_0.5_IID}, we compare original baselines (i.e., FedAvg) and the baselines combined with \SysName and momentum-based counterpart methods (denoted with \textbf{+}) on three datasets. For consistency, all experimental configurations--including the cut layer for \SysName and hyperparameter settings--are kept identical to those in Table~\ref{tab:exp_result_momentum}. Experimental results show that \SysName consistently improves the performance of baselines across all benchmarks and data distributions. Compared to the three momentum-based counterparts, \SysName yields superior global model performance in terms of both accuracy and convergence speed, although the round-to-accuracy improvements of counterpart methods are also significant in more balanced data settings. Besides, we observe that the performance gain from \SysName is more noticeable in scenarios involving more complex tasks, deeper models, and greater data imbalance. This advantage aligns well with real-world SFL deployments, where the edge devices often suffer from severe training latency for complex models, while the central server typically possesses significantly greater computational resources, making it more efficient to offload a larger portion of the model to the server for improved training performance.

\begin{table*}[t]
\centering
\caption{Performance comparison between \SysName and momentum-based counterparts across two datasets: the \textbf{Shakespeare} text benchmark and the \textbf{Google Speech} audio benchmark. Methods denoted with \textbf{+} represent baseline FedAvg combined with \SysName or its counterparts. We report the average and standard deviation of \textit{Top-1 accuracy}, the number of communication rounds (\textbf{R}) required to reach the target accuracy (i.e., 90\% of the best global model accuracy by FedAvg), and the corresponding convergence speedup (\textbf{R}$\uparrow$). All results are averaged over three trials, with \textbf{bold} indicating the best performance for each setup.}
\resizebox{0.9\linewidth}{!}{
\begin{tabular}{lccccccccc}
\toprule[1pt]
\textbf{Setup} && \multicolumn{3}{c}{\textbf{Shakespeare/Inherently Non-IID}} && \multicolumn{3}{c}{\textbf{Google Speech/DIR{\scriptsize\textbf{200}}\textbf{(0.2)}}}\\ 
\cmidrule(r){1-1} \cmidrule(r){3-5} \cmidrule(r){7-9}
\textbf{Methods}  && \textbf{Acc. (\%)} & \textbf{R} & \textbf{R}$\uparrow$ && \textbf{Acc. (\%)} & \textbf{R} & \textbf{R}$\uparrow$\\ \hline 
\textit{FedAvg}~\citep{mcmahan2017}  && 46.08\scriptsize$\pm$0.53 & 170 & 1.00$\times$  && 90.84\scriptsize $\pm$0.11 & 31 & 1.00$\times$ \\
\textbf{+} FedAvgM~\citep{hsu2019measuring} && 49.13\scriptsize$\pm$0.29 & \textbf{62} & \textbf{2.74$\times$}  && \textbf{91.14\scriptsize $\pm$0.07} & 24 & 1.29$\times$ \\
\textbf{+} SlowMo~\citep{wang2019slowmo} && 47.62\scriptsize$\pm$0.74 & 85 & 2.00$\times$  && 90.47\scriptsize $\pm$0.02 & 54 & 0.57$\times$ \\
\textbf{+} FedNAG~\citep{yang2022federated} && 42.56\scriptsize$\pm$2.59 & 210 & 0.81$\times$  && 90.46\scriptsize $\pm$0.05 & 49 & 0.63$\times$ \\
\rowcolor{gray!20} \textbf{+} \SysName && \textbf{51.83\scriptsize$\pm$0.21} & 74 &2.30$\times$  && 90.41\scriptsize$\pm$0.05 & \textbf{13} &\textbf{2.38$\times$}
\\[-0.65ex]
\bottomrule[1pt]
\end{tabular}
}
\vspace{-3mm}
\label{tab:supp_exp_result_momentum_extra_dataset}
\end{table*}

Moreover, we evaluate the performance of \SysName under varying levels of client participation on the large-scale Tiny-ImageNet benchmark, which involves 200 candidate clients. In the $n-$th round, the server randomly selects a subset of clients $\mathcal{J}^{n}$ to perform local training. Table~\ref{tab:supp_exp_result_involvement} reports results across different participation scale, with $\lvert\mathcal{J}^{n}\lvert$ ranging from 10 to 30. Increased client involvement significantly accelerates global model convergence for most methods, as evidenced by fewer communication rounds required to reach the target accuracy. For FedAvg, we observe that involving more clients in each round slows down the model convergence speed, while enhancing the accuracy in the long run. However, the accuracy improves only marginally for methods such as FedAvgM and SlowMo, underscoring the difficulty of training with a large number of diverse local updates. In contrast, \SysName consistently outperforms both the baseline FedAvg and three momentum-based counterparts--even under an extreme participation rate of just 5\%--and shows more significant performance gains as $\lvert \mathcal{J}^{n} \rvert$ increases.

\subsection{\textbf{More Benchmarks}}
\label{appendix_sub_benchmark}

In addition to image classification tasks, we evaluate \SysName and its counterparts on a text benchmark (Shakespeare) and an audio dataset (Google Speech). The Shakespeare dataset is inherently non-IID, and we randomly select 100 speaking roles from the plays and assign them to corresponding clients, following work~\citep{li2023fednar}. For Google Speech, we simulate 200 clients and assign the training samples following the Dirichlet distribution. Note that the selection ratios are 0.2 and 0.1 for Shakespeare and Google Speech, respectively. We employ a stacked transformer model~\citep{vaswani2017attention} for Shakespeare with a mini-batch size of $B=100$, and VGG-11~\citep{simonyan2014very} for Google Speech with a mini-batch size of $B=32$. The hyperparameters search for two benchmarks includes: 1) $\beta_{\mathrm{avgm}}=0.7$ for Shakespeare and $\beta_{\mathrm{avgm}}=0.5$ for Google Speech; 2) $\alpha_{\mathrm{slow}}=1$ and $\beta_{\mathrm{slow}}=0.6$ for both benchmarks; 3) all clients execute the same number of local steps when running SlowMo, FedNAG, and MergeSFL, with $T=115$ for Shakespeare and $T=15$ for Google Speech. For \SysName, we set the staleness factor $\alpha$ to $-0.1$ and the global momentum coefficient $\beta_g$ to $0.5$ for both benchmarks.

Table~\ref{tab:supp_exp_result_momentum_extra_dataset} reports the comparisons with the momentum-based methods on the Shakespeare and Google Speech datasets. On Shakespeare, \SysName significantly improves the global model accuracy over the FedAvg baseline and outperforms all momentum-based counterparts. For Google Speech, while the accuracy differences between the baseline and momentum-based methods are marginal, \SysName achieves better round-to-accuracy performance.

% Moreover, Table~\ref{tab:supp_exp_result_SFL_extra_dataset} provides results comparing \SysName with split FL methods on two datasets. On Shakespeare, \SysName improves both accuracy and convergence speed, in terms of wall-clock time, by a large margin. For Google Speech, we observe that the global model accuracy across all methods is comparable, while \SysName reaches the target accuracy with the lowest training latency.

In a nutshell, for complex tasks involving deeper models, \SysName speeds up model convergence and enhances long-term global model accuracy. For simpler tasks, \SysName maintains similarly high accuracy while significantly reducing training latency in temporal space, highlighting its efficiency and practicality in system-constrained environments.

\section{Discussions}
\label{appendix_discussion} 

\begin{table}[t]
\centering
\caption{Symbols and notations in the paper.}
\resizebox{0.95\linewidth}{!}{%
\begin{tabular}{c|l}
\noalign{\hrule height 1.5pt}
\rowcolor{gray!20}\textbf{Symbol} & \textbf{Explanation} \\ \hline
$\mathcal{J}$ & Set of clients \\
$\mathcal{J}^{n}$ & Subset of selected clients at cound $n$ \\
$\mathcal{J}^{(n, \tau)}$ & Subset of active clients at step $\tau$  \\
$\mathcal{H}^{n}$ & Historical Momentum buffers \\
$\mathcal{D}_{j}$ & Set of local samples of client $j$ \\
$\mathcal{B}_{j}$ & Mini-batch samples of client $j$ \\
$\mathcal{W}$/$\mathcal{W}_{c}/\mathcal{W}_{s}$ & Weights of global/client/server model \\
$m_{s}$/$m_{g}$  & Server-side/Global Momentum buffer \\
$\beta$/$\beta_g$ & SGDM/Global Momentum coefficient \\
$L$/$\mathbf{A}$ & Index/Activations of cut layer \\
$\hat{\mathbf{Y}}$ & Model Predictions \\
$\eta$ & Learning rate for optimizer \\
$\alpha$ & Staleness factor \\
$N$ & Number of communication rounds \\
$E$ & Number of local epochs \\
$T_{j}$ & Number of training steps of client $j$ \\
$\tau$ & Index of cut trainings step \\
$c$ & Number of task-specific classes \\
\noalign{\hrule height 1.5pt}
\end{tabular}
}
\label{tab:notation}
\end{table}

\begin{algorithm}[!t]
\DontPrintSemicolon
\small
\SetNoFillComment
  \SetKwFunction{Aggregation}{Update}
  \SetKwFunction{Alignment}{Fuse}
  \SetKwFunction{OneStep}{One-step Momentum SGD}
  \SetKwInOut{Input}{input}
  \SetKwInOut{Output}{output}
  \let\oldnl\nl
  \newcommand{\nonl}{\renewcommand{\nl}{\let\nl\oldnl}}
  \Input{
  A set of clients $\mathcal{J}$ with data $\{\mathcal{D}_{j}\}_{j \in \mathcal{J}}$; Index of cut layer $L$; Batch size $B$; Local epochs $E$; Communication rounds $N$; Learning rate $\eta$; Momentum coefficient $\beta$; Global Momentum coefficient $\beta_g$; Staleness factor $\alpha$; Initialized full model weights $\mathcal{W}^{0}$.\\
  } 
  \Output{
  $\mathcal{W}^{N}$
  }
  \BlankLine
    \For{each round $n = 1, 2, \cdots, N$}
{
    $\mathcal{J}^{n}\gets$ randomly select a subset of clients from $\mathcal{J}$\\
    $\{T_{j}\}_{j\in\mathcal{J}^{n}}\gets$ obtain local steps $E \lfloor \frac{\lvert \mathcal{D}_{j} \lvert}{B} \rfloor$ for $j\in\mathcal{J}$\\
    $\{\mathcal{W}_{c, j}^{(n, 0)}, \mathcal{W}_{s, j}^{(n, 0)} \}_{j\in\mathcal{J}^{n}}\gets\mathcal{W}^{n-1}=[\mathcal{W}^{n-1}_{c}, \mathcal{W}^{n-1}_{s}]$ split model at $L$-th layer\\
    $\bar{m}_{s}^{(n, 0)}\gets\mathbf{0}$, $\mathcal{H}^{n}\gets\emptyset$\\
    \For{step $\tau = 0, 1, \cdots, max\{T_{j}\}_{j\in\mathcal{J}^{n}} - 1$}
    {
        \For{$j\in \mathcal{J}^{(n, \tau)}$ \textbf{in parallel}}
        {
        \tcp*[h]{Client Forward Propagation}\\
        $\mathbf{A}^{(n, \tau)}_{j}\gets$ $\{f^{\mathcal{W}_{c, j}^{(n, \tau)}}(\mathbf{x})\}_{\mathbf{x}\in\mathcal\mathcal{B}^{\tau}_{j}}$, $\mathcal{B}^{\tau}_{j}\subseteq\mathcal{D}_{j}$\\
        \tcp*[h]{Server Forward Propagation}\\
        $\hat{\mathbf{Y}}\gets\{f^{\mathcal{W}_{s, j}^{(n, \tau)}}(a)\}_{a\in \mathbf{A}^{(n, \tau)}_{j}}$\\
        \tcp*[h]{Server Backpropagation}\\
        $\nabla\mathcal{L}_{\mathcal{B}^{\tau}_{j}}(\mathcal{W}_{s, j}^{(n, \tau)})\gets$ Stochastic gradients on $\mathbf{A}^{(n, \tau)}_{j}$\\
        $m_{s, j}^{(n, \tau+1)} \gets \beta\bar{m}_{s}^{(n, \tau)}+\nabla\mathcal{L}_{\mathcal{B}^{\tau}_{j}}(\mathcal{W}_{s, j}^{(n, \tau)})$ \\
        $\mathcal{W}_{s, j}^{(n, \tau+1)} \gets \mathcal{W}_{s, j}^{(n, \tau)} - \eta m_{s, j}^{(n, \tau+1)}$\\
        \tcp*[h]{Client Backpropagation}\\
        $\nabla\mathcal{L}_{\mathcal{B}^{\tau}_{j}}({\mathcal{W}_{L}}_{s, j}^{(n, \tau)})\gets$ Sending back to client $j$\\
        $\mathcal{W}_{c, j}^{(n, \tau+1)}\gets$ Updating client-side submodel $\mathcal{W}_{c, j}^{(n, \tau)}$\\
            \If{$\tau = \lvert T_{j} \lvert$}
            {
            \tcp*[h]{Historical Momentum Update}\\
            $\mathcal{H}^{n}$$\gets$ record $m_{s, j}^{(n, \tau+1)}$\\
            }
        }
        \tcp*[h]{Momentum Alignment}\\
        $\bar{m}_{s}^{(n, \tau+1)}$$\gets$ Updating($\{m_{s, j}^{(n, \tau+1)}\}_{j\in\mathcal{J}^{(n, \tau)}}$, $\mathcal{H}^{n}$, $\alpha$) by Equation~\ref{eq:approach_5}\\
    }
    \tcp*[h]{Global Model Aggregation}\\
    $\{\mathcal{W}^{n}_{j}\}_{j\in{\mathcal{J}^{n}}}\gets \{[\mathcal{W}^{(n, T_{j})}_{c, j}, \mathcal{W}^{(n, T_{j})}_{s, j}]\}_{j\in{\mathcal{J}^{n}}}$\\
    $\bar{\mathcal{W}}^{n}\gets$ Weighted Averaging of $\{\mathcal{W}^{n}_{j}\}_{j\in{\mathcal{J}^{n}}}$\\
    $m_{g}^{n}\gets\beta_{g}m_{g}^{n-1}+\mathcal{W}^{n-1}-\bar{\mathcal{W}}^{n}$\\
    $\mathcal{W}^{n}\gets\mathcal{W}^{n-1}-m_{g}^{n}$\\
}
\caption{Split FL with \SysName}
\label{algo:pseudocode_2}
\end{algorithm}

In this section, we first discuss the workflow of \SysName as a plug-in method when integrated into the split FL training paradigm. We then provide a detailed comparison with the other two split FL frameworks, i.e., SFLV1 and SFLV2.

\subsection{Split FL with \SysName}
\label{subsec:framework_workflow}

% \begin{table*}[t]
% \centering
% \caption{Symbols and notations in the paper.}
% \resizebox{0.9\linewidth}{!}{%
% \begin{tabular}{cl|cl}
% \noalign{\hrule height 1.5pt}
% \rowcolor{gray!20}\textbf{Symbol} & \textbf{Explanation} & \textbf{Symbol} & \textbf{Explanation} \\ \hline
% $\mathcal{J}$ & Set of clients & $\mathcal{J}^{n}$ & Subset of selected clients at cound $n$\\
% $\mathcal{J}^{(n, \tau)}$ & Subset of active clients at step $\tau$ & $\mathcal{H}^{n}$ & Historical Momentum buffers \\
% $\mathcal{D}_{j}$ & Set of local samples of client $j$ & $\mathcal{B}_{j}$ & Mini-batch samples of client $j$ \\
% $\mathcal{W}$ & Weights of global model & $\mathcal{W}_{c}/\mathcal{W}_{s}$ & Weights of client/server model \\
% $\mathbf{A}$ & Activations of cut layer & $\hat{\mathbf{Y}}$ & Model Predictions\\
% $l(\cdot)$ & Loss function & $\nabla(\cdot)$ & Gradients calculation \\
% $m_{s}$ & Server-side Momentum buffer & $m_{g}$ & Global Momentum buffer \\
% $\eta$ & Learning rate for optimizer & $\alpha$ & Staleness factor \\
% $\beta$ & SGDM Momentum coefficient & $\beta_g$ & Global momentum coefficient \\
% $N$ & Number of communication rounds & $E$ & Number of local epochs \\
% $T_{j}$ & Number of training steps of client $j$ & $\tau$ & Index of cut trainings step\\
% $c$ & Number of task-specific classes &  $L$ & Index of cut layer \\
% \noalign{\hrule height 1.5pt}
% \end{tabular}
% }
% \label{tab:notation}
% \end{table*}

Table~\ref{tab:notation} summarizes the notations used throughout this paper. Algorithm~\ref{algo:pseudocode_2} outlines the overall split FL workflow when integrated with \SysName. As a client-transparent method, the client-side training remains unchanged (\textbf{row 8, 13, and 14}) while \SysName introduces additional operations at the central server: it aligns the momentum buffers across all server-side optimizers at each local step (\textbf{row 17}), which then imposes constraints on the surrogate server-side models training by synchronized momentum buffer in the subsequent local step (\textbf{row 11-12}). The global client- and server-side models are updated at the end of each communication round (\textbf{row 18-21}) through weighted averaging with a global momentum term, inspired by the work~\citep{hsu2019measuring}.

\subsection{Comparison with Other Split FL Frameworks}
\label{subsec:framework_comparison}
In this section, we analyze two split FL frameworks that differ in the server-side submodel updates: \textit{Parallel updating} (e.g., SFLV1) and \textit{Sequential updating} (e.g., SFLV2). We then present the rationale behind the design of \SysName: it preserves the system efficiency by parallel updating while imposing tighter and more stable constraints on the server-side submodel updates.

% Table~\ref{tab:framework_details_appendix} provides a detailed comparison across five key aspects of split FL frameworks. 
In the \textit{Parallel updating} framework, the server updates surrogate server-side models in parallel and periodically aggregates submodels at a frequency of $\bar{\tau}$. As the results reported in Table~\ref{tab:exp_result_SFL}, the performance of this framework is highly sensitive to the choice of $\bar{\tau}$. Both client-side and server-side models in \textit{Parallel updating} require periodic synchronization. It allows the global objective to be approached by separately optimizing local objectives in parallel across clients. Based on this concept, increasing the server-side aggregation frequency (a smaller $\bar{\tau}<N$) generally yields performance gains. However, the overall performance of the full global model may still be hindered due to the asynchronous client-side model aggregation, which remains fixed at each communication round. Note that the performance of the global model in \textit{Parallel updating} has not yet been improved, maintaining the performance of it in the FL setting when $\bar{\tau}=N$. Besides, surrogate submodels are reinitialized with the aggregated global server-side model, and the corresponding SGDM optimizers are reset. This suppresses the benefits of momentum, further impacting overall performance. The workflow of \textit{Sequential updating} also has inherent limitations, such as longer training latency, which is exacerbated not only by a larger number of local epochs $E$, but also by an increase in the scale of participating clients. This is primarily because \textit{Sequential updating} replaces server-side aggregation with updating the server-side model by sequential training with clients. Moreover, since \textit{Sequential updating} maintains a single model on the server side, the structure of the server-side model is inherently fixed. It means the framework is less adaptable to varying cut layers between the client-side and server-side models, limiting its flexibility in practical deployments.

Compared to \textit{Sequential updating}, the \textit{Parallel updating} framework offers better performance in terms of system efficiency, as it fully leverages parallel training instead of merely transforming the local model training in FL into a collaborative pipeline. \SysName builds upon and improves the \textit{Parallel updating} framework in two key ways: 1) \SysName reduces the inconsistency between the client-side and corresponding surrogate server-side models within the same client, as they are synchronized at the same round-level frequency; and 2) \SysName also alleviates the inconsistency across server-side models through momentum alignment at step-level frequency. Our \SysName with momentum alignment imposes tighter constraints on the server-side training compared with (split) FL training with momentum SGD as the optimizer under the same settings.

% \begin{table}
% \centering
% \caption{Server-side Comparison between \SysName and two split FL frameworks.}
% \label{tab:framework_details_appendix}
% \resizebox{0.9\linewidth}{!}{
% \begin{tabular}{c|ccc}
% \toprule[1.5pt]
% \textbf {Comparison} & 
% \textit{\begin{tabular}[c]{@{}c@{}}Parallel \\ Updating\end{tabular}} & 
% \textit{\begin{tabular}[c]{@{}c@{}}Sequential \\ Updating\end{tabular}} &
% \textit{\SysName} \\ \hline
% \textbf{Submodel Updating}  & Parallel & Sequential & Parallel \\
% \textbf{Submodel Aggregation} & $\bar{\tau}\in[1, N]$ & \XSolidBrush & $N$ \\        
% \textbf{SGDM Optimizer} & $\bar{\tau}$-dependent & \Checkmark &\Checkmark \\                        
% \textbf{Varying $L$} & \Checkmark & \XSolidBrush &\Checkmark \\
% \textbf{Training Latency} & $\mathcal{O}(\frac{E}{\bar{\tau}})$ & $\mathcal{O}(\lvert\mathcal{J}\lvert E)$ &$\mathcal{O}(E)$ \\
% \bottomrule[1.5pt]
% \end{tabular}
% }
% \end{table}

\section{Convergence Analysis}
\label{appendix_convergence}

\subsection{Preliminaries}
As formulated in Equation~\ref{eq:approach_1}, the global objective $\mathcal{L}(\mathcal{W})$ in split FL is defined as the weighted averaging of the local objectives $\mathcal{L}_{\mathcal{D}_{j}}(\mathcal{W})=\mathcal{L}_{j}(\mathcal{W})$ across clients $\mathcal{J}$: 
\begin{equation}
    \mathcal{L}(\mathcal{W})=\sum_{j\in\mathcal{J}}p_{j}\mathcal{L}_{j}(\mathcal{W}).
    \label{eq:supp_convergence_1}
\end{equation}
Following the Assumption~\ref{main_assumption_1}, where each local objective is $L$-smooth, we can also have that the global objective is $L$-smooth. The convergence analysis of \SysName is essentially bounding the optimization error between the global model after $N$ communication rounds, $\mathcal{W}^N$, and the global optimum $\mathcal{W}^*$.

The work~\citep{han2024convergence} proposes a solution for deriving convergence guarantees in split training by decomposing the error bound into the client-side and server-side components separately.

\begin{proposition}
Under the Assumption~\ref{main_assumption_1}, the error bound in split FL can be decomposed as
% \begin{equation}
%     \mathbb{E}[\mathcal{L}(\mathcal{W}^{N})] - \mathcal{L}(\mathcal{W}^{*}) \leq 
%     \frac{L}{2}(\mathbb{E}[\Vert \mathcal{W}_{c}^{N} - \mathcal{W}_{c}^{*}\Vert^{2}_{2}] + \mathbb{E}[\Vert \mathcal{W}_{s}^{N} - \mathcal{W}_{s}^{*}\Vert^{2}_{2}]),
%     \label{eq:supp_convergence_2}
% \end{equation}

\begin{equation}
\small
\begin{aligned}
&\mathbb{E}[\mathcal{L}(\mathcal{W}^{N})] - \mathcal{L}(\mathcal{W}^{*}) \\
&\leq \frac{L}{2} \Big(
        \mathbb{E}[\Vert \mathcal{W}_{c}^{N} - \mathcal{W}_{c}^{*}\Vert^{2}_{2}]
    + \mathbb{E}[\Vert \mathcal{W}_{s}^{N} - \mathcal{W}_{s}^{*}\Vert^{2}_{2}]
    \Big),
\end{aligned}
\label{eq:supp_convergence_2}
\end{equation}
\label{supp_proposition_1}where $\mathcal{W}^{N}=[\mathcal{W}_{c}^{N}, \mathcal{W}_{s}^{N}]$ and $\mathcal{W}^{*}=[\mathcal{W}_{c}^{*}, \mathcal{W}_{s}^{*}]$ denote the client-side and server-side weights of the global model and the global optimum, respectively.
\end{proposition}

\begin{proof}
Since global objective $\mathcal{L}(\mathcal{W})$ is $L$-smooth, we have
% \begin{equation}
%     \mathcal{L}(\mathcal{W}^{N}) - \mathcal{L}(\mathcal{W}^{*}) \leq \langle \nabla\mathcal{L}(\mathcal{W}^{*}), \mathcal{W}^{N}-\mathcal{W}^{*} \rangle + \frac{L}{2}\Vert\mathcal{W}^{N}-\mathcal{W}^{*}\Vert_{2}^{2}.
%     \label{eq:supp_convergence_3}
% \end{equation}
\begin{equation}
\small
\begin{aligned}
&\mathcal{L}(\mathcal{W}^{N}) - \mathcal{L}(\mathcal{W}^{*}) \\
    &\leq \langle \nabla\mathcal{L}(\mathcal{W}^{*}), \mathcal{W}^{N}-\mathcal{W}^{*} \rangle + \frac{L}{2}\Vert\mathcal{W}^{N}-\mathcal{W}^{*}\Vert_{2}^{2}.
\end{aligned}
\label{eq:supp_convergence_3}
\end{equation}
Taking the expectation over clients, we have
% \begin{equation}
%     \mathbb{E}[\mathcal{L}(\mathcal{W}^{N})] - \mathcal{L}(\mathcal{W}^{*}) \leq \langle \nabla\mathcal{L}(\mathcal{W}^{*}), \mathbb{E}[\mathcal{W}^{N}-\mathcal{W}^{*}] \rangle + \frac{L}{2}\mathbb{E}[\Vert\mathcal{W}^{N}-\mathcal{W}^{*}\Vert_{2}^{2}].
%     \label{eq:supp_convergence_4}
% \end{equation}
\begin{equation}
\small
\begin{aligned}
    \mathbb{E}[\mathcal{L}(\mathcal{W}^{N})] - \mathcal{L}(\mathcal{W}^{*}) \\
    &\hspace{-7em}\leq \langle \nabla\mathcal{L}(\mathcal{W}^{*}), 
        \mathbb{E}[\mathcal{W}^{N}-\mathcal{W}^{*}] \rangle
    + \frac{L}{2}\mathbb{E}[\Vert\mathcal{W}^{N}-\mathcal{W}^{*}\Vert_{2}^{2}].
\end{aligned}
\label{eq:supp_convergence_4}
\end{equation}

The global optimum $\mathcal{W}^{*}$ indicates $\nabla\mathcal{L}(\mathcal{W}^{*})=0$. Therefore, we have
% \begin{equation}
%     \begin{aligned}
%     \mathbb{E}[\mathcal{L}(\mathcal{W}^{N})] - \mathcal{L}(\mathcal{W}^{*}) &\leq \frac{L}{2}\mathbb{E}[\Vert\mathcal{W}^{N}-\mathcal{W}^{*}\Vert_{2}^{2}]\\
%     &= \frac{L}{2}\mathbb{E}[\Vert[\mathcal{W}_{c}^{N}, \mathcal{W}_{s}^{N}]-[\mathcal{W}_{c}^{*}, \mathcal{W}_{s}^{*}]\Vert_{2}^{2}]\\
%     &= \frac{L}{2}(\mathbb{E}[\Vert\mathcal{W}_{c}^{N}-\mathcal{W}_{c}^{*}]\Vert_{2}^{2}] - \mathbb{E}[\Vert\mathcal{W}_{s}^{N}-\mathcal{W}_{s}^{*}]\Vert_{2}^{2}]).
%     \end{aligned}
%     \label{eq:supp_convergence_5}
% \end{equation}
\begin{equation}
\small
    \begin{aligned}
    &\mathbb{E}[\mathcal{L}(\mathcal{W}^{N})] - \mathcal{L}(\mathcal{W}^{*}) \\
    &\leq \frac{L}{2}\mathbb{E}[\Vert\mathcal{W}^{N}-\mathcal{W}^{*}\Vert_{2}^{2}]\\
    &= \frac{L}{2}\mathbb{E}[\Vert[\mathcal{W}_{c}^{N}, \mathcal{W}_{s}^{N}]-[\mathcal{W}_{c}^{*}, \mathcal{W}_{s}^{*}]\Vert_{2}^{2}]\\
    &= \frac{L}{2} \big( 
    \mathbb{E}\big[\Vert\mathcal{W}_{c}^{N} - \mathcal{W}_{c}^{*}\Vert_{2}^{2}\big] - \mathbb{E}\big[\Vert\mathcal{W}_{s}^{N} - \mathcal{W}_{s}^{*}\Vert_{2}^{2}\big] \big).
    \end{aligned}
\label{eq:supp_convergence_5}
\end{equation}
\end{proof}

The Preposition~\ref{supp_proposition_1} allows the convergence analysis of split FL to be conducted separately for the client-side and server-side submodels.

Unlike the original SFLV1, where the server periodically aggregates the surrogate submodels at a frequency of $\bar{\tau}$, \SysName synchronizes both the client-side and server-side submodels at the end of each communication round. Based on this fact, we introduce a lemma that provides an error bound of the global submodel, which can be applied to both sides of the model updates.

\begin{lemma}
    Under Assumptions~\ref{main_assumption_1}, \ref{main_assumption_2}, \ref{main_assumption_3}, and \ref{main_assumption_4}, the aggregated submodel $\hat{\mathcal{W}}^{n}$ at round $n$ converges toward its global optimum $\hat{\mathcal{W}^{*}}$. Given a small enough learning rate $\eta^{n}\leq\frac{1}{2LT}$, the following error bound holds:
    % \begin{equation}
    % \small
    % \mathbb{E}[\Vert \hat{\mathcal{W}}^{n} - \hat{\mathcal{W}}^{*}\Vert^{2}_{2}] \leq 
    % \frac{ 16 \lvert\mathcal{J}\lvert \sum_{j\in\mathcal{J}} p_{j}^{2}(2\sigma^{2}+G^{2}) }{ \mu^{2}(\gamma+n) } + \\
    % \frac{ 1536L  \sum_{j\in\mathcal{J}} p_{j}(2\sigma^{2}+G^{2}) }{ \mu^{3}(\gamma+n)(\gamma+1) } +
    % \frac{ (\gamma+1) \mathbb{E}[\Vert \hat{\mathcal{W}}^{0} - \hat{\mathcal{W}}^{*}\Vert^{2}_{2}]}{ (\gamma+n) },
    % \label{eq:supp_convergence_6}
    % \end{equation}
    \begin{equation}
    \small
    \begin{aligned}
    &\mathbb{E}[\Vert \hat{\mathcal{W}}^{n} - \hat{\mathcal{W}}^{*}\Vert^{2}_{2}] \\
    &\leq \frac{ 16 \lvert\mathcal{J}\rvert \sum_{j\in\mathcal{J}} p_{j}^{2}(2\sigma^{2}+G^{2}) }{ \mu^{2}(\gamma+n) }\\
    &\hspace{2em}+ \frac{ 1536L  \sum_{j\in\mathcal{J}} p_{j}(2\sigma^{2}+G^{2}) }
    { \mu^{3}(\gamma+n)(\gamma+1) } \\
    &\hspace{2em} + \frac{ (\gamma+1) \mathbb{E}[\Vert \hat{\mathcal{W}}^{0} - \hat{\mathcal{W}}^{*}\Vert^{2}_{2}]}{ (\gamma+n) },
    \end{aligned}
    \label{eq:supp_convergence_6}
    \end{equation}
where $\gamma=\frac{8L}{\mu}-1$ and $T$ denotes the number of local steps.
\label{supp_lemma_1}
\end{lemma}

\begin{proof}
We define the convergence error at round $n$ as $\Delta^{n}=\mathbb{E}[\Vert \hat{\mathcal{W}}^{n}-\hat{\mathcal{W}}^{*} \Vert^{2}_{2}]$. At the beginning of round $n$, the submodel is initialized with $\hat{\mathcal{W}}^{n-1}$ and updated over $T$ steps of SGD. The global submodel $\hat{\mathcal{W}}^{n}$ is then obtained by weighted aggregating: $\hat{\mathcal{W}}^{n} = \hat{\mathcal{W}}^{n-1} - \eta^{n}\sum_{j\in\mathcal{J}}p_{j}\sum_{\tau\in[T]}g_{j}^{(n, \tau)}$ where $g_{j}^{(n, \tau)}=\mathcal{L}_{\mathcal{B}_{j}^{\tau}}(\hat{\mathcal{W}}^{(n, \tau)})$ denotes the stochastic gradients of on the mini-batch $\mathcal{B}_{j}^{\tau}\subseteq\mathcal{D}_{j}$. Therefore, we have

\small
\allowdisplaybreaks
\begin{align}
&\mathbb{E}[\Vert \hat{\mathcal{W}}^{n}-\hat{\mathcal{W}}^{*} \Vert^{2}_{2}]  \nonumber \\
&= \mathbb{E}\big[\Vert \hat{\mathcal{W}}^{n-1} 
    - \eta^{n}\!\sum_{j\in\mathcal{J}}p_{j}\!\sum_{\tau\in[T]}g_{j}^{(n, \tau)} 
    -\hat{\mathcal{W}}^{*}\Vert^{2}_{2}\big] \nonumber \\
&= \mathbb{E}\big[\Vert \hat{\mathcal{W}}^{n-1} -\hat{\mathcal{W}}^{*} 
    - \eta^{n}\!\sum_{\tau\in[T]}\nabla\mathcal{L}(\hat{\mathcal{W}}^{(n-1, \tau)}) \nonumber \\
&\hspace{2em} + \eta^{n}\!\sum_{\tau\in[T]}\nabla\mathcal{L}(\hat{\mathcal{W}}^{(n-1, \tau)}) 
    - \eta^{n}\!\sum_{j\in\mathcal{J}}p_{j}\!\sum_{\tau\in[T]}g_{j}^{(n, \tau)} 
    \Vert^{2}_{2}\big] \nonumber \\
&= \mathbb{E}\big[\Vert \hat{\mathcal{W}}^{n-1} -\hat{\mathcal{W}}^{*} 
    - \eta^{n}\!\sum_{\tau\in[T]}\nabla\mathcal{L}(\hat{\mathcal{W}}^{(n-1, \tau)}) \Vert^{2}_{2}\big] \nonumber \\
&\hspace{2em} + \mathbb{E}\big[\Vert \eta^{n}\!\sum_{\tau\in[T]}\nabla\mathcal{L}(\hat{\mathcal{W}}^{(n-1, \tau)}) 
    - \eta^{n}\!\sum_{j\in\mathcal{J}}p_{j}\!\sum_{\tau\in[T]}g_{j}^{(n, \tau)} \Vert^{2}_{2}\big] \nonumber \\
&\hspace{2em} + 2\eta^{n}\big\langle\mathbb{E}[\hat{\mathcal{W}}^{n-1} -\hat{\mathcal{W}}^{*} 
    - \eta^{n}\!\sum_{\tau\in[T]}\nabla\mathcal{L}(\hat{\mathcal{W}}^{(n-1, \tau)})], \nonumber \\
&\hspace{5em} \mathbb{E}[\sum_{\tau\in[T]}\nabla\mathcal{L}(\hat{\mathcal{W}}^{(n-1, \tau)}) 
    - \sum_{j\in\mathcal{J}}p_{j}\!\sum_{\tau\in[T]}g_{j}^{(n, \tau)}]\big\rangle \nonumber \\
&= \mathbb{E}\big[\Vert \hat{\mathcal{W}}^{n-1} -\hat{\mathcal{W}}^{*} 
    - \eta^{n}\!\sum_{\tau\in[T]}\nabla\mathcal{L}(\hat{\mathcal{W}}^{(n-1, \tau)}) \Vert^{2}_{2}\big] \nonumber \\
&\hspace{2em} + \mathbb{E}\big[\Vert \eta^{n}\!\sum_{\tau\in[T]}\nabla\mathcal{L}(\hat{\mathcal{W}}^{(n-1, \tau)}) 
    - \eta^{n}\!\sum_{j\in\mathcal{J}}p_{j}\!\sum_{\tau\in[T]}g_{j}^{(n, \tau)} \Vert^{2}_{2}\big],
\label{eq:supp_convergence_7}
\end{align}
where the last equality uses the facts: 1) unbiased stochastic gradients $\mathbb{E}_{\tau\sim[T]}[g_{j}^{(n, \tau)}]=\nabla\mathcal{L}_{j}$; and 2) the global objective is the weighted averaging of local objectives in Equation~\ref{eq:supp_convergence_1}, therefore $\mathbb{E}[\sum_{\tau\in[T]}\nabla\mathcal{L}(\hat{\mathcal{W}}^{(n-1, \tau)}) - \sum_{j\in\mathcal{J}}p_{j}\sum_{\tau\in[T]}g_{j}^{(n, \tau)}]=0$.

We first bound the first term on the right-hand side of Equation~\ref{eq:supp_convergence_7}.
% \begin{equation}
%     \small
%     \begin{aligned}
%     &\mathbb{E}[\Vert \hat{\mathcal{W}}^{n-1} -\hat{\mathcal{W}}^{*} - \eta^{n}\sum_{\tau\in[T]}\nabla\mathcal{L}(\hat{\mathcal{W}}^{(n-1, \tau)}) \Vert^{2}_{2}] \\
%     &= \mathbb{E}[\Vert \hat{\mathcal{W}}^{n-1} -\hat{\mathcal{W}}^{*} \Vert^{2}_{2}] 
%     - 2\eta^{n}\mathbb{E}[\langle \hat{\mathcal{W}}^{n-1} -\hat{\mathcal{W}}^{*}, \sum_{\tau\in[T]}\nabla\mathcal{L}(\hat{\mathcal{W}}^{(n-1, \tau)}) \rangle] \\
%     &+ (\eta^{n})^{2}\mathbb{E}[\Vert \sum_{\tau\in[T]}\nabla\mathcal{L}(\hat{\mathcal{W}}^{(n-1, \tau)}) \Vert^{2}_{2}]\\
%     &= \mathbb{E}[\Vert \hat{\mathcal{W}}^{n-1} -\hat{\mathcal{W}}^{*} \Vert^{2}_{2}] 
%     - 2\eta^{n}\mathbb{E}[\langle \hat{\mathcal{W}}^{n-1} -\hat{\mathcal{W}}^{*}, \sum_{j\in{\mathcal{J}}}\sum_{\tau\in[T]}p_j\nabla\mathcal{L}_{j}(\hat{\mathcal{W}}^{(n-1, \tau)}) \rangle] \\
%     &+ (\eta^{n})^{2}\mathbb{E}[\Vert \sum_{j\in{\mathcal{J}}}\sum_{\tau\in[T]}p_j\nabla\mathcal{L}_{j}(\hat{\mathcal{W}}^{(n-1, \tau)}) \Vert^{2}_{2}],
%     \label{eq:supp_convergence_8}
%     \end{aligned}
% \end{equation}
\begin{equation}
\small
\begin{aligned}
    &\mathbb{E}[\Vert \hat{\mathcal{W}}^{n-1} - \hat{\mathcal{W}}^{*} - \eta^{n}\sum_{\tau\in[T]}\nabla\mathcal{L}(\hat{\mathcal{W}}^{(n-1, \tau)}) \Vert^{2}_{2}] \\
    &= \mathbb{E}[\Vert \hat{\mathcal{W}}^{n-1} -\hat{\mathcal{W}}^{*} \Vert^{2}_{2}] \\
    &\hspace{2em}- 2\eta^{n}\mathbb{E}[\langle \hat{\mathcal{W}}^{n-1} -\hat{\mathcal{W}}^{*}, \sum_{\tau\in[T]}\nabla\mathcal{L}(\hat{\mathcal{W}}^{(n-1, \tau)}) \rangle]\\
    &\hspace{2em}+ (\eta^{n})^{2}\mathbb{E}[\Vert \sum_{\tau\in[T]}\nabla\mathcal{L}(\hat{\mathcal{W}}^{(n-1, \tau)}) \Vert^{2}_{2}]\\
    &= \mathbb{E}[\Vert \hat{\mathcal{W}}^{n-1} -\hat{\mathcal{W}}^{*} \Vert^{2}_{2}] \\
    &\hspace{2em}- 2\eta^{n}\mathbb{E}[\langle \hat{\mathcal{W}}^{n-1} -\hat{\mathcal{W}}^{*}, \sum_{j\in{\mathcal{J}}}\sum_{\tau\in[T]}p_j\nabla\mathcal{L}_{j}(\hat{\mathcal{W}}^{(n-1, \tau)}) \rangle] \\
    &\hspace{2em}+ (\eta^{n})^{2}\mathbb{E}[\Vert \sum_{j\in{\mathcal{J}}}\sum_{\tau\in[T]}p_j\nabla\mathcal{L}_{j}(\hat{\mathcal{W}}^{(n-1, \tau)}) \Vert^{2}_{2}],
\label{eq:supp_convergence_8}
\end{aligned}
\end{equation}
where we expand the term $\nabla\mathcal{L}(\hat{\mathcal{W}}^{(n-1, \tau)})=\sum_{j\in{\mathcal{J}}}p_j\nabla\mathcal{L}_{j}(\hat{\mathcal{W}}^{(n-1, \tau)})$. 

For the second term in the Equation~\ref{eq:supp_convergence_8}, we have:
\allowdisplaybreaks
    \small
    \begin{align}
    &-2\eta^{n}\mathbb{E}[\langle \hat{\mathcal{W}}^{n-1} -\hat{\mathcal{W}}^{*}, \sum_{j\in{\mathcal{J}}}\sum_{\tau\in[T]}p_j\nabla\mathcal{L}_{j}(\hat{\mathcal{W}}^{(n-1, \tau)}) \rangle]\nonumber \\
    &=-2\eta^{n}\sum_{j\in{\mathcal{J}}}\sum_{\tau\in[T]}p_j \mathbb{E}[\langle \hat{\mathcal{W}}^{n-1} -\hat{\mathcal{W}}^{*}, \nabla\mathcal{L}_{j}(\hat{\mathcal{W}}^{(n-1, \tau)}) \rangle] \nonumber \\
    &\leq -2\eta^{n}\sum_{j\in{\mathcal{J}}}\sum_{\tau\in[T]}p_j\mathbb{E}[\mathcal{L}_{j}(\hat{\mathcal{W}}^{n-1}) - \mathcal{L}_{j}(\hat{\mathcal{W}}^{*}) \nonumber \\ 
    &\hspace{2em}+\frac{\mu}{4}\Vert \hat{\mathcal{W}}^{n-1} - \hat{\mathcal{W}}^{*} \Vert^{2}_{2} - L\Vert \hat{\mathcal{W}}^{(n-1, \tau)} - \hat{\mathcal{W}}^{n-1} \Vert^{2}_{2}] \nonumber \\
    &=-2\eta^{n}T\mathbb{E}[\sum_{j\in{\mathcal{J}}}p_j\mathcal{L}_{j}(\hat{\mathcal{W}}^{n-1}) - \sum_{j\in{\mathcal{J}}}p_j\mathcal{L}_{j}(\hat{\mathcal{W}}^{*})] \nonumber \\
    &\hspace{2em}-\frac{\eta^{n}\mu T}{2}\sum_{j\in{\mathcal{J}}}p_j\mathbb{E}[\Vert \hat{\mathcal{W}}^{n-1} - \hat{\mathcal{W}}^{*} \Vert^{2}_{2}] \nonumber \\
    &\hspace{2em}+2\eta^{n}L\sum_{j\in{\mathcal{J}}}\sum_{\tau\in[T]}p_j\mathbb{E}[\Vert \hat{\mathcal{W}}^{(n-1, \tau)} - \hat{\mathcal{W}}^{n-1} \Vert^{2}_{2}] \nonumber \\
    &=-2\eta^{n}T\mathbb{E}[\mathcal{L}(\hat{\mathcal{W}}^{n-1}) - \mathcal{L}(\hat{\mathcal{W}}^{*})] \nonumber \\
    &\hspace{2em}-\frac{\eta^{n}\mu T}{2}\mathbb{E}[\Vert \hat{\mathcal{W}}^{n-1}-\hat{\mathcal{W}}^{*} \Vert^{2}_{2}] \nonumber \\
    &\hspace{2em}+2\eta^{n}L\sum_{j\in{\mathcal{J}}}\sum_{\tau\in[T]}p_j\mathbb{E}[\Vert \hat{\mathcal{W}}^{(n-1, \tau)} - \hat{\mathcal{W}}^{n-1} \Vert^{2}_{2}] \nonumber \\
    &\leq -2\eta^{n}T\mathbb{E}[\mathcal{L}(\hat{\mathcal{W}}^{n-1}) - \mathcal{L}(\hat{\mathcal{W}}^{*})] \nonumber \\
    &\hspace{2em} -\frac{\eta^{n}\mu T}{2}\mathbb{E}[\Vert \hat{\mathcal{W}}^{n-1}-\hat{\mathcal{W}}^{*} \Vert^{2}_{2}] \nonumber \nonumber \\
    &\hspace{2em} +24(2\sigma^{2}+ G^{2})(\eta^{n})^{3}T^{3}L, 
    \label{eq:supp_convergence_9}
    \end{align} 
where we employ \textit{Lemma 5} in work~\citep{karimireddy2020scaffold} for the first inequality, and \textit{Lemma C.5} in work~\citep{han2024convergence} for the second inequality. Besides, we also use the fact $\sum_{j\in{\mathcal{J}}}p_j=1$.

For the last term in the Equation~\ref{eq:supp_convergence_8}, we have:
% \begin{equation}
%     \small
%     \begin{aligned}
%     &(\eta^{n})^{2}\mathbb{E}[\Vert \sum_{j\in{\mathcal{J}}}\sum_{\tau\in[T]}p_j\nabla\mathcal{L}_{j}(\hat{\mathcal{W}}^{(n-1, \tau)}) \Vert^{2}_{2}]\\
%     &\leq (\eta^{n})^{2}T\lvert\mathcal{J}\lvert \sum_{j\in{\mathcal{J}}}\sum_{\tau\in[T]}p_{j}^{2}\mathbb{E}[\Vert \nabla\mathcal{L}_{j}(\hat{\mathcal{W}}^{(n-1, \tau)})\Vert^{2}_{2}]\\
%     &=(\eta^{n})^{2}T\lvert\mathcal{J}\lvert \sum_{j\in{\mathcal{J}}}\sum_{\tau\in[T]}p_{j}^{2}\mathbb{E}[\Vert \nabla\mathcal{L}_{j}(\hat{\mathcal{W}}^{(n-1, \tau)}) - g_{j}^{(n,\tau)} + g_{j}^{(n,\tau)} \Vert^{2}_{2}]\\
%     &=(\eta^{n})^{2}T\lvert\mathcal{J}\lvert \sum_{j\in{\mathcal{J}}}p_{j}^{2}\sum_{\tau\in[T]}
%     (\mathbb{E}[\Vert \nabla\mathcal{L}_{j}(\hat{\mathcal{W}}^{(n-1, \tau)}) - g_{j}^{(n,\tau)} \Vert^{2}_{2}] + \mathbb{E}[\Vert g_{j}^{(n,\tau)} \Vert^{2}_{2}])\\
%     &\leq (\sigma^{2} + G^{2})(\eta^{n})^{2}T^{2}\lvert\mathcal{J}\lvert\sum_{j\in{\mathcal{J}}}p_{j}^{2}.
%     \end{aligned}
%     \label{eq:supp_convergence_10}
% \end{equation}
\begin{equation}
    \small
    \begin{aligned}
    &(\eta^{n})^{2}\mathbb{E}[\Vert \sum_{j\in{\mathcal{J}}}\sum_{\tau\in[T]}p_j\nabla\mathcal{L}_{j}(\hat{\mathcal{W}}^{(n-1, \tau)}) \Vert^{2}_{2}]\\
    &\leq (\eta^{n})^{2}T\lvert\mathcal{J}\lvert \sum_{j\in{\mathcal{J}}}\sum_{\tau\in[T]}p_{j}^{2}\mathbb{E}[\Vert \nabla\mathcal{L}_{j}(\hat{\mathcal{W}}^{(n-1, \tau)})\Vert^{2}_{2}]\\
    &=(\eta^{n})^{2}T\lvert\mathcal{J}\lvert \sum_{j\in{\mathcal{J}}}\sum_{\tau\in[T]}p_{j}^{2}\mathbb{E}[\Vert \nabla\mathcal{L}_{j}(\hat{\mathcal{W}}^{(n-1, \tau)}) - \\
    &\hspace{2em}g_{j}^{(n,\tau)} + g_{j}^{(n,\tau)} \Vert^{2}_{2}]\\
    &=(\eta^{n})^{2}T\lvert\mathcal{J}\lvert \sum_{j\in{\mathcal{J}}}p_{j}^{2}\sum_{\tau\in[T]}
    (\mathbb{E}[\Vert \nabla\mathcal{L}_{j}(\hat{\mathcal{W}}^{(n-1, \tau)}) - g_{j}^{(n,\tau)} \Vert^{2}_{2}]\\
    &\hspace{2em}+ \mathbb{E}[\Vert g_{j}^{(n,\tau)} \Vert^{2}_{2}])\\
    &\leq (\sigma^{2} + G^{2})(\eta^{n})^{2}T^{2}\lvert\mathcal{J}\lvert\sum_{j\in{\mathcal{J}}}p_{j}^{2}.
    \end{aligned}
    \label{eq:supp_convergence_10}
\end{equation}

The first inequality follows from Cauchy-Schwarz inequality; the second equality holds due to the vanishing cross-term, since $\mathbb{E}[\nabla\mathcal{L}_{j}(\hat{\mathcal{W}}^{(n-1, \tau)}) - g_{j}^{(n,\tau)}]=0$; and the last inequality is based on the bounded variance and gradients in Assumptions~\ref{main_assumption_3} and \ref{main_assumption_4}.

By substituting Equation~\ref{eq:supp_convergence_9} and Equation~\ref{eq:supp_convergence_10} into Equation~\ref{eq:supp_convergence_8}, we obtain the following bound:
% \begin{equation}
%     \small
%     \begin{aligned}
%     &\mathbb{E}[\Vert \hat{\mathcal{W}}^{n-1} -\hat{\mathcal{W}}^{*} - \eta^{n}\sum_{\tau\in[T]}\nabla\mathcal{L}(\hat{\mathcal{W}}^{(n-1, \tau)}) \Vert^{2}_{2}] \\
%     &\leq  \mathbb{E}[\Vert \hat{\mathcal{W}}^{n-1} -\hat{\mathcal{W}}^{*} \Vert^{2}_{2}]\\
%     &-2\eta^{n}T\mathbb{E}[\mathcal{L}(\hat{\mathcal{W}}^{n-1}) - \mathcal{L}(\hat{\mathcal{W}}^{*})]
%     -\frac{\eta^{n}\mu T}{2}\mathbb{E}[\Vert \hat{\mathcal{W}}^{n-1} - \hat{\mathcal{W}}^{*} \Vert^{2}_{2}]
%     +24(2\sigma^{2} + G^{2})(\eta^{n})^{3}T^{3}L\\
%     &+ (\sigma^{2} + G^{2})(\eta^{n})^{2}T^{2}\lvert\mathcal{J}\lvert\sum_{j\in{\mathcal{J}}}p_{j}^{2} \\
%     &= (1-\frac{\eta^{n}\mu T}{2})\mathbb{E}[\Vert \hat{\mathcal{W}}^{n-1} - \hat{\mathcal{W}}^{*} \Vert^{2}_{2}] -2\eta^{n}T\mathbb{E}[\mathcal{L}(\hat{\mathcal{W}}^{n-1}) - \mathcal{L}(\hat{\mathcal{W}}^{*})]\\
%     &+24(2\sigma^{2} + G^{2})(\eta^{n})^{3}T^{3}L + (\sigma^{2} + G^{2})(\eta^{n})^{2}T^{2}\lvert\mathcal{J}\lvert\sum_{j\in{\mathcal{J}}}p_{j}^{2}.
%     \label{eq:supp_convergence_11}
%     \end{aligned}
% \end{equation}
\allowdisplaybreaks
    \small
    \begin{align}
    &\mathbb{E}[\Vert \hat{\mathcal{W}}^{n-1} -\hat{\mathcal{W}}^{*} - \eta^{n}\sum_{\tau\in[T]}\nabla\mathcal{L}(\hat{\mathcal{W}}^{(n-1, \tau)}) \Vert^{2}_{2}] \nonumber \\
    &\leq  \mathbb{E}[\Vert \hat{\mathcal{W}}^{n-1} -\hat{\mathcal{W}}^{*} \Vert^{2}_{2}] \nonumber \\
    &\hspace{2em}-2\eta^{n}T\mathbb{E}[\mathcal{L}(\hat{\mathcal{W}}^{n-1})- \mathcal{L}(\hat{\mathcal{W}}^{*})] \nonumber \\
    &\hspace{2em}-\frac{\eta^{n}\mu T}{2}\mathbb{E}[\Vert \hat{\mathcal{W}}^{n-1} - \hat{\mathcal{W}}^{*} \Vert^{2}_{2}] \nonumber \\
    &\hspace{2em}+24(2\sigma^{2} + G^{2})(\eta^{n})^{3}T^{3}L \nonumber \\
    &\hspace{2em}+ (\sigma^{2} + G^{2})(\eta^{n})^{2}T^{2}\lvert\mathcal{J}\lvert\sum_{j\in{\mathcal{J}}}p_{j}^{2} \nonumber \\
    &= (1-\frac{\eta^{n}\mu T}{2})\mathbb{E}[\Vert \hat{\mathcal{W}}^{n-1}- \hat{\mathcal{W}}^{*} \Vert^{2}_{2}] \nonumber \\
    &\hspace{2em} -2\eta^{n}T\mathbb{E}[\mathcal{L}(\hat{\mathcal{W}}^{n-1}) - \mathcal{L}(\hat{\mathcal{W}}^{*})] \nonumber \\
    &\hspace{2em}+24(2\sigma^{2} + G^{2})(\eta^{n})^{3}T^{3}L \nonumber \\
    &\hspace{2em}+ (\sigma^{2} + G^{2})(\eta^{n})^{2}T^{2}\lvert\mathcal{J}\lvert\sum_{j\in{\mathcal{J}}}p_{j}^{2}.
    \label{eq:supp_convergence_11}
    \end{align}

We then bound the second term on the right-hand side of Equation~\ref{eq:supp_convergence_7}.
% \begin{equation}
%     \small
%     \begin{aligned}
%     &\mathbb{E}[\Vert \eta^{n}\sum_{\tau\in[T]}\nabla\mathcal{L}(\hat{\mathcal{W}}^{(n-1, \tau)}) - \eta^{n}\sum_{j\in\mathcal{J}}p_{j}\sum_{\tau\in[T]}g_{j}^{(n, \tau)} \Vert^{2}_{2}]\\
%     &=(\eta^{n})^{2}\mathbb{E}[\Vert \sum_{j\in\mathcal{J}}p_{j}\sum_{\tau\in[T]}\nabla\mathcal{L}_{j}(\hat{\mathcal{W}}^{(n-1, \tau)}) - \sum_{j\in\mathcal{J}}p_{j}\sum_{\tau\in[T]}g_{j}^{(n, \tau)}\Vert^{2}_{2}]\\
%     &=(\eta^{n})^{2}\mathbb{E}[\Vert \sum_{j\in\mathcal{J}}\sum_{\tau\in[T]}p_{j}(\nabla\mathcal{L}_{j}(\hat{\mathcal{W}}^{(n-1, \tau)})-g_{j}^{(n, \tau)})\Vert^{2}_{2}]\\
%     &\leq(\eta^{n})^{2}T\sum_{\tau\in[T]}\mathbb{E}[\Vert\sum_{j\in\mathcal{J}}p_{j}(\nabla\mathcal{L}_{j}(\hat{\mathcal{W}}^{(n-1, \tau)})-g_{j}^{(n, \tau)})\Vert^{2}_{2}]\\
%     &\leq(\eta^{n})^{2}T\lvert\mathcal{J}\lvert\sum_{\tau\in[T]}\sum_{j\in\mathcal{J}}\mathbb{E}[\Vert p_{j}(\nabla\mathcal{L}_{j}(\hat{\mathcal{W}}^{(n-1, \tau)})-g_{j}^{(n, \tau)})\Vert^{2}_{2}]\\
%     &\leq(\eta^{n})^{2}T\lvert\mathcal{J}\lvert\sum_{\tau\in[T]}\sum_{j\in\mathcal{J}}p_{j}^{2}\sigma^{2}\\
%     &=(\eta^{n})^{2}T^{2}\sigma^{2}\lvert\mathcal{J}\lvert\sum_{j\in\mathcal{J}}p_{j}^{2},
%     \end{aligned}
%     \label{eq:supp_convergence_12}
% \end{equation}
\allowdisplaybreaks
    \small
    \begin{align}
    &\mathbb{E}[\Vert \eta^{n}\sum_{\tau\in[T]}\nabla\mathcal{L}(\hat{\mathcal{W}}^{(n-1, \tau)}) - \eta^{n}\sum_{j\in\mathcal{J}}p_{j}\sum_{\tau\in[T]}g_{j}^{(n, \tau)} \Vert^{2}_{2}] \nonumber \\
    &=(\eta^{n})^{2}\mathbb{E}[\Vert \sum_{j\in\mathcal{J}}p_{j}\sum_{\tau\in[T]}\nabla\mathcal{L}_{j}(\hat{\mathcal{W}}^{(n-1, \tau)})- \nonumber \\
    &\hspace{2em}\sum_{j\in\mathcal{J}}p_{j}\sum_{\tau\in[T]}g_{j}^{(n, \tau)}\Vert^{2}_{2}] \nonumber \\
    &=(\eta^{n})^{2}\mathbb{E}[\Vert \sum_{j\in\mathcal{J}}\sum_{\tau\in[T]}p_{j}(\nabla\mathcal{L}_{j}(\hat{\mathcal{W}}^{(n-1, \tau)})-g_{j}^{(n, \tau)})\Vert^{2}_{2}] \nonumber \\
    &\leq(\eta^{n})^{2}T\sum_{\tau\in[T]}\mathbb{E}[\Vert\sum_{j\in\mathcal{J}}p_{j}(\nabla\mathcal{L}_{j}(\hat{\mathcal{W}}^{(n-1, \tau)})-g_{j}^{(n, \tau)})\Vert^{2}_{2}] \nonumber \\
    &\leq(\eta^{n})^{2}T\lvert\mathcal{J}\lvert\sum_{\tau\in[T]}\sum_{j\in\mathcal{J}}\mathbb{E}[\Vert p_{j}(\nabla\mathcal{L}_{j}(\hat{\mathcal{W}}^{(n-1, \tau)})-g_{j}^{(n, \tau)})\Vert^{2}_{2}] \nonumber \\
    &\leq(\eta^{n})^{2}T\lvert\mathcal{J}\lvert\sum_{\tau\in[T]}\sum_{j\in\mathcal{J}}p_{j}^{2}\sigma^{2} \nonumber \\
    &=(\eta^{n})^{2}T^{2}\sigma^{2}\lvert\mathcal{J}\lvert\sum_{j\in\mathcal{J}}p_{j}^{2},
    \label{eq:supp_convergence_12}
    \end{align}  
where we apply the Cauchy-Schwarz inequality to the first and second inequalities, and use the bounded variance in Assumption~\ref{main_assumption_3} to obtain the last inequality.

By substituting Equation~\ref{eq:supp_convergence_11} and Equation~\ref{eq:supp_convergence_12}, the convergence error in Equation~\ref{eq:supp_convergence_7} can be bounded as follows:
% \begin{equation}
%     \small
%     \begin{aligned}
%     &\Delta^{n}=\mathbb{E}[\Vert \hat{\mathcal{W}}^{n}-\hat{\mathcal{W}}^{*} \Vert^{2}_{2}]\\
%     &\leq(1-\frac{\eta^{n}\mu T}{2})\mathbb{E}[\Vert \hat{\mathcal{W}}^{n-1} - \hat{\mathcal{W}}^{*} \Vert^{2}_{2}] -2\eta^{n}T\mathbb{E}[\mathcal{L}(\hat{\mathcal{W}}^{n-1}) - \mathcal{L}(\hat{\mathcal{W}}^{*})]\\
%     &+24(2\sigma^{2} + G^{2})(\eta^{n})^{3}T^{3}L + (\sigma^{2} + G^{2})(\eta^{n})^{2}T^{2}\lvert\mathcal{J}\lvert\sum_{j\in{\mathcal{J}}}p_{j}^{2}\\
%     &+(\eta^{n})^{2}T^{2}\sigma^{2}\lvert\mathcal{J}\lvert\sum_{j\in\mathcal{J}}p_{j}^{2}\\
%     &=(1-\frac{\eta^{n}\mu T}{2})\Delta^{n-1} -2\eta^{n}T\mathbb{E}[\mathcal{L}(\hat{\mathcal{W}}^{n-1}) - \mathcal{L}(\hat{\mathcal{W}}^{*})]\\
%     &+24(2\sigma^{2} + G^{2})(\eta^{n})^{3}T^{3}L + (2\sigma^{2} + G^{2})(\eta^{n})^{2}T^{2}\lvert\mathcal{J}\lvert\sum_{j\in{\mathcal{J}}}p_{j}^{2}\\
%     &\leq (1-\frac{\eta^{n}\mu T}{2})\Delta^{n-1} + 24(2\sigma^{2} + G^{2})(\eta^{n})^{3}T^{3}L + (2\sigma^{2} + G^{2})(\eta^{n})^{2}T^{2}\lvert\mathcal{J}\lvert\sum_{j\in{\mathcal{J}}}p_{j}^{2}.
%     \end{aligned}  
%     \label{eq:supp_convergence_13}
% \end{equation}
\begin{equation}
    \small
    \begin{aligned}
    &\Delta^{n}=\mathbb{E}[\Vert \hat{\mathcal{W}}^{n}-\hat{\mathcal{W}}^{*} \Vert^{2}_{2}]\\
    &\leq(1-\frac{\eta^{n}\mu T}{2})\mathbb{E}[\Vert \hat{\mathcal{W}}^{n-1} - \hat{\mathcal{W}}^{*} \Vert^{2}_{2}] \\
    &\hspace{2em}-2\eta^{n}T\mathbb{E}[\mathcal{L}(\hat{\mathcal{W}}^{n-1}) - \mathcal{L}(\hat{\mathcal{W}}^{*})]\\
    &\hspace{2em}+24(2\sigma^{2} + G^{2})(\eta^{n})^{3}T^{3}L + (\sigma^{2} + G^{2})(\eta^{n})^{2}T^{2}\lvert\mathcal{J}\lvert\sum_{j\in{\mathcal{J}}}p_{j}^{2}\\
    &\hspace{2em}+(\eta^{n})^{2}T^{2}\sigma^{2}\lvert\mathcal{J}\lvert\sum_{j\in\mathcal{J}}p_{j}^{2}\\
    &=(1-\frac{\eta^{n}\mu T}{2})\Delta^{n-1} -2\eta^{n}T\mathbb{E}[\mathcal{L}(\hat{\mathcal{W}}^{n-1}) - \mathcal{L}(\hat{\mathcal{W}}^{*})]\\
    &\hspace{2em}+24(2\sigma^{2} + G^{2})(\eta^{n})^{3}T^{3}L \\
    &\hspace{2em}+ (2\sigma^{2} + G^{2})(\eta^{n})^{2}T^{2}\lvert\mathcal{J}\lvert\sum_{j\in{\mathcal{J}}}p_{j}^{2}\\
    &\leq (1-\frac{\eta^{n}\mu T}{2})\Delta^{n-1} + 24(2\sigma^{2} + G^{2})(\eta^{n})^{3}T^{3}L \\
    &\hspace{2em}+ (2\sigma^{2} + G^{2})(\eta^{n})^{2}T^{2}\lvert\mathcal{J}\lvert\sum_{j\in{\mathcal{J}}}p_{j}^{2}.
    \end{aligned}  
    \label{eq:supp_convergence_13}
\end{equation}

Suppose the learning rate is small enough $\eta^{n}=\frac{2\beta}{T(\gamma+n-1)}$, where $\beta=\frac{2}{\mu}$, $\gamma=\frac{8L}{\mu}-1$, such that there exists $v=max\{ 
    \frac{16 \lvert\mathcal{J}\lvert \sum_{j\in\mathcal{J}} p_{j}^{2}(2\sigma^{2}+G^{2})}{\mu^{2}}+
    \frac{1536L  \sum_{j\in\mathcal{J}} p_{j}(2\sigma^{2}+G^{2})}{\mu^{3}(\gamma +1)},
    (\gamma+1) \mathbb{E}[\Vert \hat{\mathcal{W}}^{0} - \hat{\mathcal{W}}^{*}\Vert^{2}_{2}]
    \}$, satisfying $\Delta^{n}\leq\frac{v}{\gamma+n}$. Assuming this bound holds for round $n$, we can show that it also holds for round $n + 1$. The Equation~\ref{eq:supp_convergence_13} gives the bound for $n+1$ as

% \begin{equation}
%     \small
%     \begin{aligned}
%     \Delta^{n+1}& \leq (1-\frac{\eta^{n+1}\mu T}{2})\Delta^{n} + 24(2\sigma^{2} + G^{2})(\eta^{n+1})^{3}T^{3}L + (2\sigma^{2} + G^{2})(\eta^{n+1})^{2}T^{2}\lvert\mathcal{J}\lvert\sum_{j\in{\mathcal{J}}}p_{j}^{2}\\
%     &\leq (1-\frac{\mu\beta}{\gamma+n})\Delta^{n}\\
%     &\leq (1-\frac{\mu\beta}{\gamma+n})\frac{v}{\gamma+n}\\
%     &= \frac{\gamma+n-2}{(\gamma+n)^2}v\\
%     &\leq \frac{\gamma+n-1}{(\gamma+n)^2}v\\
%     &\leq \frac{v}{\gamma+n+1},
%     \end{aligned}
%     \label{eq:supp_convergence_14}
% \end{equation}
\allowdisplaybreaks
    \small
    \begin{align}
    &\Delta^{n+1}\leq (1-\frac{\eta^{n+1}\mu T}{2})\Delta^{n} + 24(2\sigma^{2} + G^{2})(\eta^{n+1})^{3}T^{3}L \nonumber\\
    &\hspace{4em}+ (2\sigma^{2} + G^{2})(\eta^{n+1})^{2}T^{2}\lvert\mathcal{J}\lvert\sum_{j\in{\mathcal{J}}}p_{j}^{2} \nonumber \\
    &\hspace{2.5em}\leq (1-\frac{\mu\beta}{\gamma+n})\Delta^{n} \nonumber \\
    &\hspace{2.5em}\leq (1-\frac{\mu\beta}{\gamma+n})\frac{v}{\gamma+n} \nonumber \\
    &\hspace{2.5em}= \frac{\gamma+n-2}{(\gamma+n)^2}v \nonumber \\
    &\hspace{2.5em}\leq \frac{\gamma+n-1}{(\gamma+n)^2}v \nonumber \\
    &\hspace{2.5em}\leq \frac{v}{\gamma+n+1},
    \label{eq:supp_convergence_14}
    \end{align} 
where we use the $(\gamma+n-1)(\gamma+n+1)\leq(\gamma+n)^{2}$ for the last inequality.

The Equation~\ref{eq:supp_convergence_14} indicates that it also holds for $n+1$. Therefore, we have
% \begin{equation}
%     \small
%     \begin{aligned}
%     \mathbb{E}[\Vert \hat{\mathcal{W}}^{n} - \hat{\mathcal{W}}^{*}\Vert^{2}_{2}]&\leq\frac{v}{\gamma+n}\\
%     &= \frac{max\{ 
%     \frac{16 \lvert\mathcal{J}\lvert \sum_{j\in\mathcal{J}} p_{j}^{2}(2\sigma^{2}+G^{2})}{\mu^{2}}+
%     \frac{1536L  \sum_{j\in\mathcal{J}} p_{j}(2\sigma^{2}+G^{2})}{\mu^{3}(\gamma +1)},
%     (\gamma+1) \mathbb{E}[\Vert \hat{\mathcal{W}}^{0} - \hat{\mathcal{W}}^{*}\Vert^{2}_{2}]
%     \}}{\gamma+n}\\
%     &\leq 
%     \frac{ 16 \lvert\mathcal{J}\lvert \sum_{j\in\mathcal{J}} p_{j}^{2}(2\sigma^{2}+G^{2}) }{ \mu^{2}(\gamma+n) } +
%     \frac{ 1536L  \sum_{j\in\mathcal{J}} p_{j}(2\sigma^{2}+G^{2}) }{ \mu^{3}(\gamma+n)(\gamma+1) } +
%     \frac{ (\gamma+1) \mathbb{E}[\Vert \hat{\mathcal{W}}^{0} - \hat{\mathcal{W}}^{*}\Vert^{2}_{2}]}{ (\gamma+n) },
%     \end{aligned}  
%     \label{eq:supp_convergence_15}
% \end{equation}
\begin{equation}
    \small
    \begin{aligned}
    &\mathbb{E}[\Vert \hat{\mathcal{W}}^{n} - \hat{\mathcal{W}}^{*}\Vert^{2}_{2}]\\
    &\leq\frac{v}{\gamma+n}\\
    &= \frac{1}{\gamma+n} max\{ 
    \frac{16 \lvert\mathcal{J}\lvert \sum_{j\in\mathcal{J}} p_{j}^{2}(2\sigma^{2}+G^{2})}{\mu^{2}}+\\
    &\hspace{2em}\frac{1536L  \sum_{j\in\mathcal{J}} p_{j}(2\sigma^{2}+G^{2})}{\mu^{3}(\gamma +1)},
    (\gamma+1) \mathbb{E}[\Vert \hat{\mathcal{W}}^{0} - \hat{\mathcal{W}}^{*}\Vert^{2}_{2}]
    \}\\
    &\leq 
    \frac{ 16 \lvert\mathcal{J}\lvert \sum_{j\in\mathcal{J}} p_{j}^{2}(2\sigma^{2}+G^{2}) }{ \mu^{2}(\gamma+n) } +
    \frac{ 1536L  \sum_{j\in\mathcal{J}} p_{j}(2\sigma^{2}+G^{2}) }{ \mu^{3}(\gamma+n)(\gamma+1) }\\
    &\hspace{2em}+\frac{ (\gamma+1) \mathbb{E}[\Vert \hat{\mathcal{W}}^{0} - \hat{\mathcal{W}}^{*}\Vert^{2}_{2}]}{ (\gamma+n) },
    \end{aligned}  
    \label{eq:supp_convergence_15}
\end{equation}
\end{proof}

\subsection{Proof of Theorem~\ref{main_theorem_1}}

To discuss the convergence guarantee of \SysName under partial client participation, we first formulate the client selection process before the proof of Theorem~\ref{main_theorem_1}. 

At each communication round $n$, the server randomly selects a subset of clients with a fixed participation rate $\theta \in (0, 1]$. We can formulate client selection as
\begin{equation}
    \small
    \delta_{j}^{n}=
    \begin{cases}
        0 & (j\not\in\mathcal{J}^{n}) \\
        1 & (j\in\mathcal{J}^{n}),
    \end{cases}
    \label{eq:supp_convergence_16}
\end{equation}
where $\delta_{j}^{n}\sim Ber(\theta)$ follows a Bernoulli distribution: the probability of $j$-th client participating in round $n$ is $s_{j}=Pr(\delta_{j}^{n}=1)=\theta, \forall j\in\mathcal{J}, n\in[N]$. For instance, we set $\theta=0.2$ in evaluations across CIFAR10, CIFAR100, Tiny-ImageNet, and Shakespeare benchmarks.

Under the partial participation, the update rule of the submodel $\hat{\mathcal{W}}^{n}$ in round $n$ is then reformulated as
\begin{equation}
    \small
    \begin{aligned}
    \hat{\mathcal{W}}^{n} &= \hat{\mathcal{W}}^{n-1} - \eta^{n}\sum_{j\in\mathcal{J}^{n}}p_{j}\sum_{\tau\in[T]}g_{j}^{(n, \tau)}\\
    &=\hat{\mathcal{W}}^{n-1} - \eta^{n}\sum_{j\in\mathcal{J}}\frac{p_{j}\delta_{j}^{n}}{\theta}\sum_{\tau\in[T]}g_{j}^{(n, \tau)}.
    \end{aligned} 
    \label{eq:supp_convergence_17}
\end{equation}

Suppose that $\Omega$ is the submodel updated under full participation, therefore, Equation~\ref{eq:supp_convergence_13} holds
% \begin{equation}
%     \small
%     \begin{aligned}
%     &\Delta^{n}=\mathbb{E}[\Vert \Omega^{n}-\hat{\mathcal{W}}^{*} \Vert^{2}_{2}]\\
%     &\leq (1-\frac{\eta^{n}\mu T}{2})\Delta^{n-1} + 24(2\sigma^{2} + G^{2})(\eta^{n})^{3}T^{3}L + (2\sigma^{2} + G^{2})(\eta^{n})^{2}T^{2}\lvert\mathcal{J}\lvert\sum_{j\in{\mathcal{J}}}p_{j}^{2}.
%     \end{aligned}  
%     \label{eq:supp_convergence_18}
% \end{equation}
\begin{equation}
    \small
    \begin{aligned}
    &\Delta^{n}=\mathbb{E}[\Vert \Omega^{n}-\hat{\mathcal{W}}^{*} \Vert^{2}_{2}]\\
    &\leq (1-\frac{\eta^{n}\mu T}{2})\Delta^{n-1} + 24(2\sigma^{2} + G^{2})(\eta^{n})^{3}T^{3}L \\
    &\hspace{2em}+ (2\sigma^{2} + G^{2})(\eta^{n})^{2}T^{2}\lvert\mathcal{J}\lvert\sum_{j\in{\mathcal{J}}}p_{j}^{2}.
    \end{aligned}  
    \label{eq:supp_convergence_18}
\end{equation}

We define the gap between the submodel updated under partial participation $\hat{\mathcal{W}}^{n}$ and that under full participation $\Omega^{n}$.
\allowdisplaybreaks
    \small
    \begin{align}
    \mathbb{E}[\Vert \hat{\mathcal{W}}^{n} - \Omega^{n} \Vert^{2}_{2}] 
    &=\mathbb{E}[\Vert \hat{\mathcal{W}}^{n} - \hat{\mathcal{W}}^{n-1} + \hat{\mathcal{W}}^{n-1} - \Omega^{n}  \Vert^{2}_{2}] \nonumber \\
    &\leq\mathbb{E}[\Vert \hat{\mathcal{W}}^{n} - \hat{\mathcal{W}}^{n-1}\Vert^{2}_{2}] \nonumber \\
    &=\mathbb{E}[\Vert\eta^{n}\sum_{j\in\mathcal{J}}\frac{p_{j}\delta_{j}^{n}}{\theta}\sum_{\tau\in[T]}g_{j}^{(n, \tau)}\Vert^{2}_{2}] \nonumber \\
    &\leq\frac{1}{\theta^{2}}(\eta^{n})^{2}\lvert\mathcal{J}\lvert T\sum_{j\in\mathcal{J}}\sum_{\tau\in[T]}\mathbb{E}[\Vert\delta_{j}^{n}p_{j}g_{j}^{(n, \tau)}\Vert^{2}_{2}] \nonumber \\
    &\leq\frac{1}{\theta^{2}}(\eta^{n})^{2}\lvert\mathcal{J}\lvert T\sum_{j\in\mathcal{J}}\sum_{\tau\in[T]}p_{j}^{2}G^{2}\mathbb{E}[\Vert\delta_{j}^{n}\Vert^{2}_{2}] \nonumber \\
    &=\frac{1}{\theta^{2}}(\eta^{n})^{2}\lvert\mathcal{J}\lvert T^{2}\sum_{j\in\mathcal{J}}p_{j}^{2}G^{2}\theta \nonumber \\
    &=\frac{1}{\theta}(\eta^{n})^{2}G^{2}\lvert\mathcal{J}\lvert T^{2}\sum_{j\in\mathcal{J}}p_{j}^{2},
    \label{eq:supp_convergence_19}
    \end{align}
where we use the following facts: 1) $\mathbb{E}[\hat{\mathcal{W}}^{n}-\hat{\mathcal{W}}^{n-1}]=\Omega^{n} - \hat{\mathcal{W}}^{n-1}$ for the first inequality; 2) Cauchy-Schwarz inequality for the second inequality; 3) Equation~\ref{eq:supp_convergence_17} for the second equality; and 4) for the Bernoulli distribution, we have $\mathbb{E}[\delta_{j}^{n}]=\mathbb{E}[\Vert\delta_{j}^{n}\Vert^{2}_{2}]=\theta$.

Therefore, we obtain the convergence error bound, similar to Equation~\ref{eq:supp_convergence_13}, now extended to the case of partial client participation.

% \begin{equation}
%     \small
%     \begin{aligned}
%     \Delta^{n}&=\mathbb{E}[\Vert \hat{\mathcal{W}}^{n}-\hat{\mathcal{W}}^{*} \Vert^{2}_{2}]\\
%     &=\mathbb{E}[\Vert \hat{\mathcal{W}}^{n}-\Omega^{n}+\Omega^{n}-\hat{\mathcal{W}}^{*} \Vert^{2}_{2}]\\
%     &\leq\mathbb{E}[\Vert \hat{\mathcal{W}}^{n}-\Omega^{n} \Vert^{2}_{2}] + \mathbb{E}[\Vert\Omega^{n}-\hat{\mathcal{W}}^{*} \Vert^{2}_{2}]\\
%     &\leq (1-\frac{\eta^{n}\mu T}{2})\Delta^{n-1} + 24(2\sigma^{2} + G^{2})(\eta^{n})^{3}T^{3}L \\
%     &+ (2\sigma^{2} + G^{2})(\eta^{n})^{2}T^{2}\lvert\mathcal{J}\lvert\sum_{j\in{\mathcal{J}}}p_{j}^{2} + \frac{1}{\theta}G^{2}(\eta^{n})^{2} T^{2}\lvert\mathcal{J}\lvert\sum_{j\in\mathcal{J}}p_{j}^{2}.
%     \end{aligned}  
%     \label{eq:supp_convergence_20}
% \end{equation}
\begin{equation}
    \small
    \begin{aligned}
    \Delta^{n}&=\mathbb{E}[\Vert \hat{\mathcal{W}}^{n}-\hat{\mathcal{W}}^{*} \Vert^{2}_{2}]\\
    &=\mathbb{E}[\Vert \hat{\mathcal{W}}^{n}-\Omega^{n}+\Omega^{n}-\hat{\mathcal{W}}^{*} \Vert^{2}_{2}]\\
    &\leq\mathbb{E}[\Vert \hat{\mathcal{W}}^{n}-\Omega^{n} \Vert^{2}_{2}] + \mathbb{E}[\Vert\Omega^{n}-\hat{\mathcal{W}}^{*} \Vert^{2}_{2}]\\
    &\leq (1-\frac{\eta^{n}\mu T}{2})\Delta^{n-1} + 24(2\sigma^{2} + G^{2})(\eta^{n})^{3}T^{3}L \\
    &\hspace{2em}+ (2\sigma^{2} + G^{2})(\eta^{n})^{2}T^{2}\lvert\mathcal{J}\lvert\sum_{j\in{\mathcal{J}}}p_{j}^{2} \\
    &\hspace{2em}+ \frac{1}{\theta}G^{2}(\eta^{n})^{2} T^{2}\lvert\mathcal{J}\lvert\sum_{j\in\mathcal{J}}p_{j}^{2}.
    \end{aligned}  
    \label{eq:supp_convergence_20}
\end{equation}

Similar to Equation~\ref{eq:supp_convergence_15}, we extend Lemma~\ref{supp_lemma_1} to the setting with partial client participation.
% \begin{equation}
%     \small
%     \begin{aligned}
%     \Delta^{n}&=\mathbb{E}[\Vert \hat{\mathcal{W}}^{n} - \hat{\mathcal{W}}^{*}\Vert^{2}_{2}]\\
%     &\leq 
%     \frac{ 16 \lvert\mathcal{J}\lvert \sum_{j\in\mathcal{J}} p_{j}^{2}(2\sigma^{2}+G^{2}+\frac{G^{2}}{\theta}) }{ \mu^{2}(\gamma+n) } +
%     \frac{ 1536L  \sum_{j\in\mathcal{J}} p_{j}(2\sigma^{2}+G^{2}) }{ \mu^{3}(\gamma+n)(\gamma+1) } +
%     \frac{ (\gamma+1) \mathbb{E}[\Vert \hat{\mathcal{W}}^{0} - \hat{\mathcal{W}}^{*}\Vert^{2}_{2}]}{ (\gamma+n)}.\\
%     \end{aligned}  
%     \label{eq:supp_convergence_21}
% \end{equation}
\begin{equation}
    \small
    \begin{aligned}
    \Delta^{n}&=\mathbb{E}[\Vert \hat{\mathcal{W}}^{n} - \hat{\mathcal{W}}^{*}\Vert^{2}_{2}]\\
    &\leq \frac{ 16 \lvert\mathcal{J}\lvert \sum_{j\in\mathcal{J}} p_{j}^{2}(2\sigma^{2}+G^{2}+\frac{G^{2}}{\theta}) }{ \mu^{2}(\gamma+n) } \\
    &\hspace{2em}+\frac{ 1536L  \sum_{j\in\mathcal{J}} p_{j}(2\sigma^{2}+G^{2}) }{ \mu^{3}(\gamma+n)(\gamma+1) } \\
    &\hspace{2em}+\frac{ (\gamma+1) \mathbb{E}[\Vert \hat{\mathcal{W}}^{0} - \hat{\mathcal{W}}^{*}\Vert^{2}_{2}]}{ (\gamma+n)}.
    \end{aligned}  
    \label{eq:supp_convergence_21}
\end{equation}

Recalling the Proposition~\ref{supp_proposition_1}, and substituting the submodel $\hat{\mathcal{W}}^{n}$ in Equation~\ref{eq:supp_convergence_21} with the client-side submodel $\mathcal{W}_{c}^{N}$ and server-side submodel $\mathcal{W}_{s}^{N}$ at $N$-th communication round, we have Theorem~\ref{main_theorem_1}.
\allowdisplaybreaks
    \small
    \begin{align}
    &\mathbb{E}[\mathcal{L}(\mathcal{W}^{N})] - \mathcal{L}(\mathcal{W}^{*}) \nonumber \\
    &\leq \frac{L}{2}(\mathbb{E}[\Vert \mathcal{W}_{c}^{N} - \mathcal{W}_{c}^{*}\Vert^{2}_{2}] + \mathbb{E}[\Vert \mathcal{W}_{s}^{N} - \mathcal{W}_{s}^{*}\Vert^{2}_{2}]) \nonumber \\
    &\leq \frac{ 16 \lvert\mathcal{J}\lvert L \sum_{j\in\mathcal{J}} p_{j}^{2}(2\sigma^{2}+G^{2}+\frac{G^{2}}{\theta}) }{ \mu^{2}(\gamma+n) } \nonumber \\
    &\hspace{2em}+\frac{ 1536 L^{2} \sum_{j\in\mathcal{J}} p_{j}(2\sigma^{2}+G^{2}) }{ \mu^{3}(\gamma+n)(\gamma+1) } \nonumber \\
    &\hspace{2em}+\frac{ (\gamma+1)L\mathbb{E}[\Vert \hat{\mathcal{W}}^{0} - \hat{\mathcal{W}}^{*}\Vert^{2}_{2}]}{ (\gamma+n)} \nonumber \nonumber \\
    &\leq \mathcal{O}(\frac{A} {(\gamma+N)}) + \mathcal{O}(\frac{B} {(\gamma+N)}) + \mathcal{O}(\frac{C} {{(\gamma+N)}}).
    \label{eq:supp_convergence_22}
    \end{align} 
We use the $\mathcal{O}$ to swallow all constants. The $A$, $B$, $C$ in the error bound follows $A = \lvert\mathcal{J}\lvert \sum_{j\in\mathcal{J}}p_{j}^{2}(2\sigma^{2}+(1+\frac{1}{\theta})G^{2})$, $B = \sum_{j\in\mathcal{J}}p_{j}(2\sigma^{2}+G^{2})$, $C = \Vert \mathcal{W}^{0} - \mathcal{W}^{*} \Vert$.

\end{document}